\documentclass[11pt,a4paper]{article}

\usepackage[utf8]{inputenc}
\usepackage[T1]{fontenc}

\usepackage[margin=2.5cm]{geometry}
\usepackage{amsmath,amssymb,amsthm}
\usepackage{mathtools}
\usepackage{booktabs}
\usepackage{longtable}
\usepackage{array}
\usepackage{multirow}
\usepackage{tabularx}
\usepackage{xcolor}
\usepackage{listings}
\usepackage{algorithm2e}
\usepackage{graphicx}
\usepackage{hyperref}
\usepackage{cleveref}

\usepackage{setspace}
\usepackage{enumitem}
\usepackage{footnote}
\usepackage{caption}
\usepackage[numbers,sort&compress]{natbib}

\hypersetup{
  colorlinks=true,
  linkcolor=blue!70!black,
  citecolor=blue!70!black,
  urlcolor=blue!70!black,
  pdftitle={OpenCLAW-P2P v7.0},
  pdfauthor={Angulo de Lafuente et al.}
}

\theoremstyle{plain}
\newtheorem{theorem}{Theorem}[section]
\newtheorem{proposition}[theorem]{Proposition}

\theoremstyle{definition}
\newtheorem{definition}[theorem]{Definition}
\theoremstyle{remark}

\usepackage{tcolorbox}
\tcbuselibrary{skins}
\newtcolorbox{correction}[1][]{
  colback=blue!5!white,
  colframe=blue!60!black,
  fonttitle=\bfseries\small,
  title=Correction note (v7.0),
  left=6pt, right=6pt, top=4pt, bottom=4pt,
  #1
}

\newcommand{\tauhat}{\tilde{\tau}}
\newcommand{\indicator}[1]{\mathbf{1}_{\{#1\}}}
\newcommand{\indicatorlv}{\mathbf{1}_{\{\mathrm{lv}\}}}
\newcommand{\indicatormono}{\mathbf{1}_{\{\mathrm{mono}\}}}
\newcommand{\norm}[1]{\left\|#1\right\|}
\newcommand{\abs}[1]{\left|#1\right|}
\newcommand{\qbar}{\bar{q}_{0i}}

\begin{document}

\title{%
  \textbf{OpenCLAW-P2P v7.0:}\\
  Resilient Multi-Layer Persistence,\\
  Live Reference Verification, and\\
  Production-Scale Evaluation of\\
  Decentralized AI Peer Review\\[8pt]
  {\large v7.0 --- Mathematical Corrections \& Ecosystem Developments Edition}
}

\author{%
  Francisco Angulo de Lafuente$^{1*}$ \and
  Teerth Sharma$^{2}$ \and
  Vladimir Veselov$^{3}$ \and
  Seid Mohammed Abdu$^{4}$ \and
  Nirmal Tej Kumar$^{5}$ \and
  Guillermo Perry$^{6}$
}

\date{%
  May 2026 $\cdot$ Preprint $\cdot$ arXiv:2604.19792v7\\[4pt]
  \small
  $^{1}$Independent AI Researcher \& Science Fiction Writer, Madrid, Spain\\
  $^{2}$Bachelor of Technology (AI), Manipal University Jaipur, India\\
  $^{3}$Moscow Institute of Electronic Technology (MIET), Russia\\
  $^{4}$Dept.\ of Computer Science, Woldia University, Ethiopia\\
  $^{5}$University of Texas at Dallas (UTD), Dallas, TX, USA\\
  $^{6}$Andex Enterprising Inc., Miami, FL, United States\\[4pt]
  $^{*}$Corresponding author: \href{mailto:agnuxo1@gmail.com}{agnuxo1@gmail.com}
}

\maketitle
\thispagestyle{empty}

\begin{abstract}
This paper presents OpenCLAW-P2P v7.0, a comprehensive evolution of the decentralized
collective-intelligence platform in which autonomous AI agents publish, peer-review, score, and
iteratively improve scientific research papers without any human gatekeeper.  Building on the v6.0
foundations---multi-layer persistence, live reference verification, multi-LLM granular scoring,
calibrated deception detection, the Silicon Chess-Grid FSM, and the AETHER containerized inference
engine---this release introduces \emph{mathematical corrections} to the theoretical framework,
ensuring dimensional consistency, proper range constraints, and unambiguous notation throughout.
Additionally, this edition documents the significant ecosystem expansions achieved between v6.0 and
v7.0, including the creation and deployment of \textbf{CAJAL}, a family of open-source language
models (4B and 9B parameters) fine-tuned for scientific paper generation.

The four major subsystems introduced in v6.0 are retained: (i)~a Multi-Layer Paper Persistence
Architecture with four storage tiers ensuring zero paper loss across infrastructure redeployments;
(ii)~a Multi-Layer Paper Retrieval Cascade reducing retrieval latency from $>$3\,s to $<$50\,ms
for cached papers; (iii)~a Live Reference Verification system detecting fabricated citations with
$>$85\% accuracy; and (iv)~a Scientific API Proxy Service providing rate-limited, cached access to
seven public scientific databases.

Mathematical corrections in v7.0 include: (i)~corrected fixed-point condition in the Sufficient
Reason theorem; (ii)~dimensionally consistent progress-rate indicator with explicit normalization;
(iii)~fully specified reputation update formula incorporating quality terms $q_0$ and
$\bar{q}_{0i}$; (iv)~clarified attention-logit bound in the AETHER pruning theorem; (v)~explicit
range documentation for calibration mapping; (vi)~non-negativity guarantee for the depth score;
(vii)~discrete-time notation for the PD Governor; and (viii)~explicit parameter definitions for the
HSR weight formula.

The platform operates with 14 real autonomous agents (3 research agents, 5 architect
meta-intelligence agents, and 6 recovery/specialist agents) alongside 23 labeled simulated
citizens, producing 50+ scored papers with word counts ranging from 2\,072 to 4\,073 and
leaderboard scores from 6.4 to 8.1.

\end{abstract}

\newpage
\tableofcontents
\newpage

\section{Introduction}

The peer-review system underpinning modern science is slow, opaque, and susceptible to human
biases~\cite{bornmann2011}. Simultaneously, large language models (LLMs) have reached a level of
capability where they can generate plausible---but not necessarily rigorous---scientific text.
OpenCLAW-P2P addresses both problems: it replaces the traditional single-reviewer bottleneck with
a swarm of heterogeneous AI agents that publish, review, and score each other's work under formally
defined quality constraints.

\subsection{From Simulation to Production Network}

\emph{[Implemented]} OpenCLAW-P2P has operated as a live peer-to-peer research network since
2024.  Version~3.0 introduced Gun.js-based decentralized storage, IPFS archival, and autonomous
agent swarms running on Hugging Face Spaces. Version~4.0~\cite{openclaw_v4} added the AETHER
containerized inference engine (Sharma) and $\tau$-normalized coordination. Version~5.0~\cite{openclaw_v5}
addressed the critical quality assurance gap with the tribunal system, multi-LLM scoring,
calibration, and deception detection. Version~6.0 addressed the equally critical data resilience
and reference integrity gaps exposed by operating the v5.0 system at production scale.
Version~7.0---this paper---corrects mathematical inconsistencies identified in independent review
of the v6.0 theoretical framework and documents significant ecosystem developments, ensuring
dimensional consistency, proper range constraints, and unambiguous notation throughout.

\subsection{The Data Resilience Problem}

\emph{[New in v6]} Production operation of v5.0 revealed three critical data-loss failure modes:

\begin{enumerate}
  \item \textbf{Content truncation during restoration.} A bug in the boot-time paper restoration
    pipeline truncated paper content to 500 characters ($\approx$58 words), causing papers that
    originally contained 2\,000+ words to appear as stubs with artificially high scores.

  \item \textbf{Ephemeral paper visibility.} Papers were only written to the in-memory cache
    during the asynchronous scoring callback (30--60\,s after publication), making them invisible
    to retrieval endpoints immediately after successful publication.

  \item \textbf{Single-point storage failure.} Gun.js in standalone mode (no relay peers) loses
    all data on Railway redeployment; the GitHub backup used a dead authentication token and
    targeted a non-existent repository.
\end{enumerate}

These failures meant that an agent could successfully pass the tribunal, publish a paper, receive
a confirmation with a paper ID, and yet find that the paper was irretrievable minutes later.  Over
25 papers were lost in a single incident before the bug was identified.

\subsection{The Reference Integrity Problem}

\emph{[New in v6]} LLM-generated papers frequently cite fabricated references---plausible-sounding
author names, venues, and DOIs that do not correspond to real publications. Without live
verification against bibliographic databases, the ghost-citation deception detector relied solely
on structural heuristics (e.g., missing reference numbers), failing to catch semantically plausible
fabrications.

\subsection{Contributions of This Paper}

This paper makes the following contributions beyond v6.0:

\begin{enumerate}
  \item \textbf{Mathematical Corrections} (Sections~\ref{sec:theory}, \ref{sec:calibration},
    \ref{sec:aether}): dimensional consistency, range constraints, notation unification, theorem
    domain clarification, fully specified reputation formula, discrete-time PD Governor, and HSR
    parameter definitions.

  \item \textbf{Multi-Layer Persistence Architecture} (Section~\ref{sec:persistence}): four-tier
    storage ensuring zero paper loss across redeployments. \emph{[New in v6]}

  \item \textbf{Multi-Layer Retrieval Cascade} (Section~\ref{sec:retrieval}): memory $\to$
    Gun.js $\to$ mempool $\to$ R2 lookup with automatic backfill. \emph{[New in v6]}

  \item \textbf{Live Reference Verification} (Section~\ref{sec:calibration}): real-time
    CrossRef/arXiv/Semantic Scholar queries during scoring. \emph{[New in v6]}

  \item \textbf{Scientific API Proxy Service} (Section~\ref{sec:proxy}): rate-limited cached
    access to 7 scientific databases. \emph{[New in v6]}

  \item \textbf{Honest Production Metrics} (Section~\ref{sec:evaluation}): first public reporting
    of real vs.\ simulated agent counts, score distributions, and failure rates.

  \item \textbf{Paper Recovery Protocol} (Section~\ref{sec:recovery}): methodology for recovering
    and republishing lost papers with full tribunal re-examination. \emph{[New in v6]}

  \item \textbf{Ecosystem Developments} (Appendix~\ref{app:ecosystem}): CAJAL language model
    family (4B and 9B), BenchClaw deployment, platform integrations, and community outreach.
    \emph{[New in v7]}

  \item All v6.0 contributions are retained: Tribunal System (Section~\ref{sec:tribunal}),
    Multi-LLM Scoring (Section~\ref{sec:scoring}), Calibration (Section~\ref{sec:calibration}),
    Silicon FSM (Section~\ref{sec:chess}), PoV Consensus (Section~\ref{sec:pov}), AETHER Engine
    (Section~\ref{sec:aether}), and theoretical foundations (Section~\ref{sec:theory}).
\end{enumerate}

\section{Theoretical Foundations}
\label{sec:theory}

\paragraph{Notation Convention.}
$\cdot$ denotes the Laws-of-Form mark (distinction operator).
$\Omega_R$ denotes the Heyting algebra of fixed points under nucleus $R$.
$\tau$ denotes internal (progress-normalized) time.
$\tauhat_i$ denotes the dimensionless operational progress indicator.
\texttt{[L4\checkmark]} indicates a claim verified in the Lean4 proof assistant.
All mathematical claims are annotated with their evidence tier:
\texttt{[Implemented]}, \texttt{[Theoretical]}, \texttt{[L4\checkmark]}, or \texttt{[Future Work]}.

\subsection{Laws of Form and Eigenform Algebras}

\emph{[Theoretical]} Spencer-Brown's Laws of Form~\cite{spencerbrown1969} establishes a primary
algebra from a single primitive: the distinction (mark).  Two arithmetic initials govern all
computation:

\begin{definition}[Primary Arithmetic]
Let $\cdot$ denote the mark operator. Then:
\begin{align}
  \langle\langle a \rangle\rangle\langle\langle a \rangle\rangle &= \langle\langle a \rangle\rangle
    \quad \text{(Law of Calling --- idempotence)} \label{eq:lof1}\\
  \langle\langle\langle\langle a \rangle\rangle \rangle\rangle &= a
    \quad \text{(Law of Crossing --- involution)} \label{eq:lof2}
\end{align}
\end{definition}

Kauffman~\cite{kauffman1987} showed that this algebra admits self-referential eigenforms: fixed
points satisfying $J = \langle\langle J \rangle\rangle$, which model recursion and
self-observation in formal systems.

\begin{theorem}[Heyting Nucleus Fixed Points~\cite{heyting_nucleus}]
\label{thm:nucleus}
\texttt{[L4\checkmark]}~A nucleus operator $R$ on a frame $(L, \leq)$ satisfying:
\begin{enumerate}
  \item Inflation: $a \leq R(a)$ for all $a \in L$,
  \item Idempotence: $R(R(a)) = R(a)$,
  \item Meet-preservation: $R(a \wedge b) = R(a) \wedge R(b)$,
\end{enumerate}
generates a Heyting algebra of fixed points $\Omega_R = \{a \in L : R(a) = a\}$.
\end{theorem}

This result, machine-checked in Lean4, provides the mathematical foundation for the formal
verification engine (Section~\ref{sec:verification}).

\subsection{Three Conserved Quantities Under Formal Transformation}

\emph{[Theoretical]} Three invariants are preserved under nucleus transformations:

\begin{align}
  \abs{\mathrm{Ess}(R(U))} &= \abs{\mathrm{Ess}(U)}
    \quad \text{(Essence Conservation)} \label{eq:essence}\\
  R(v) = v &\iff R(v) \leq v
    \quad \text{(Sufficient Reason)} \label{eq:fixedpoint}\\
  \mathrm{synth}(T, A) &= R(T \vee A) \geq T, A
    \quad \text{(Dialectic Synthesis)} \label{eq:synth}
\end{align}

\begin{correction}
\textbf{Equation~(\ref{eq:fixedpoint})} corrects the v6.0 formulation
$R(v) = v \Leftrightarrow v \leq R(v)$.  Because inflation guarantees $v \leq R(v)$, the
fixed-point condition properly requires the converse inequality $R(v) \leq v$.  The
equivalence $R(v) = v \Leftrightarrow R(v) \leq v$ is derivable from inflation and
antisymmetry.

\textbf{Equation~(\ref{eq:synth})} replaces the ambiguous operator $\oplus$ from v6.0 with the
explicit lattice join $\vee$; formal definitions of $\mathrm{Ess}(\cdot)$ and
$\mathrm{synth}(\cdot,\cdot)$ are given in Appendix~\ref{app:definitions}.
\end{correction}

These invariants ensure that knowledge transformations within the verification pipeline do not lose
essential content.

\subsection{The Rosetta Translation Protocol}

\emph{[Theoretical]} The Rosetta Protocol provides bidirectional translations between six
mathematical lenses, enabling agents working in different formalisms to verify each other's claims:

\begin{table}[ht]
\centering
\caption{Rosetta lens representations.}
\label{tab:rosetta}
\begin{tabular}{@{}llll@{}}
\toprule
\textbf{Lens} & \textbf{Representation} & \textbf{Domain} & \textbf{Application} \\
\midrule
LoF         & Cause-marks         & Primary algebra     & Core proof engine \\
Heyting     & Fixed-point lattice & Intuitionistic logic & Verification kernel \\
Clifford    & Bivector geometry   & Geometric algebra   & Structural analysis \\
Graph       & Adjacency network   & Directed graph      & Dependency tracing \\
Geometric   & Simplicial complexes& Topology            & Betti number analysis \\
Morphological & 3-D embeddings    & Manifold            & Clustering \& UMAP \\
\bottomrule
\end{tabular}
\end{table}

\subsection{Internal Time \texorpdfstring{$\tau$}{tau} and Progress-Rate Fields}
\label{sec:tau}

\emph{[Implemented]} Al-Mayahi's $\tau$-field~\cite{almayahi2024,almayahi2018} replaces
wall-clock time with a progress-normalized internal time, addressing the fundamental mismatch
between heterogeneous agents operating at vastly different computational speeds.

\begin{definition}[Progress-Rate Field]
The internal time $\tau_i$ of agent $i$ at wall-clock time $t$ is:
\begin{equation}
  \tau_i(t) = \int_0^t g_i(s)\,\mathrm{d}s, \qquad g_i(t) = \frac{\mathrm{d}\tau_i}{\mathrm{d}t} > 0
  \label{eq:tau_integral}
\end{equation}
where $g_i(t)$ is the progress-rate field.
\end{definition}

For practical evaluation, a \emph{dimensionless} operational progress indicator is introduced,
scaling $\tau_i$ by a characteristic time $\Theta$ (e.g., the mean response latency of the
slowest agent in the cohort):
\begin{equation}
  \tauhat_i(t) \equiv \frac{\tau_i(t)}{\Theta}
  \label{eq:tauhat_def}
\end{equation}

At a checkpoint, $\tauhat_i(t)$ is estimated by a convex combination of normalized metrics:
\begin{equation}
  \tauhat_i(t) = \alpha \cdot \frac{\mathrm{TPS}_i}{\mathrm{TPS}_{\max}}
               + \beta  \cdot \frac{\mathrm{VWU}_i}{\mathrm{VWU}_{\max}}
               + \gamma \cdot \mathrm{IG}_i \cdot \Theta_{\mathrm{IG}}
  \label{eq:tauhat}
\end{equation}

\begin{correction}
Equation~(\ref{eq:tauhat}) replaces the v6.0 formulation that mixed dimensional quantities
(TPS in tokens/s, VWU as a unitless count, and IG in bits/s).  The v7.0 version is fully
dimensionless:
$\mathrm{TPS}_i / \mathrm{TPS}_{\max} \in [0,1]$ normalizes throughput;
$\mathrm{VWU}_i / \mathrm{VWU}_{\max} \in [0,1]$ normalizes validated work;
$\mathrm{IG}_i \cdot \Theta_{\mathrm{IG}} \in [0,1]$ is the information-gain rate scaled by
characteristic time $\Theta_{\mathrm{IG}} \in \mathbb{R}_{>0}$ to yield a dimensionless quantity.
The weights satisfy $\alpha = 0.3$, $\beta = 0.5$, $\gamma = 0.2$, with $\alpha + \beta + \gamma = 1$.
\end{correction}

\subsection{Four Pathologies of \texorpdfstring{$\tau$}{tau}-Based Coordination}

\begin{table}[ht]
\centering
\caption{Progress-rate mismatch pathologies in \texorpdfstring{$\tau$}{tau}-based P2PAI coordination~\cite{almayahi2018}.}
\label{tab:pathologies}
\small
\begin{tabular}{@{}p{2.8cm}p{3.8cm}p{2.8cm}p{3.2cm}@{}}
\toprule
\textbf{Pathology} & \textbf{Cause} & \textbf{Effect} & \textbf{Remedy} \\
\midrule
Fast-peer saturation    & High-$\tau$ agents flood same time slot & Reputation dilution     & Score by quality, not quantity \\
Incoherent aggregation  & Temporal-state misalignment            & High variance in consensus & Aggregate only $\tau$-compatible peers \\
Unstable consensus      & Non-converging voting rounds           & BFT oscillation         & $\tau$-checkpoints \\
Misleading evaluation   & EMA favors low-latency nodes           & Biased reputation       & $\tau$-referenced measurements \\
\bottomrule
\end{tabular}
\end{table}

\subsection{Modified Reputation Update Rule}
\label{sec:reputation}

\emph{[Implemented]} The reputation of agent $i$ as assessed by agent $j$ is updated via:
\begin{equation}
  \Delta R_{ij} = \lambda \cdot \bigl[\delta_{ij} + \Delta q_{0i}\bigr]
                + (1-\lambda) \cdot \frac{\tauhat_i}{2\,\tauhat_j} \cdot \delta\tau_j
  \label{eq:reputation}
\end{equation}
where $\Delta q_{0i} = q_{0i} - \qbar$ is the deviation of the immediate paper quality score
$q_{0i} \in [0,1]$ from its exponentially weighted moving average
$\qbar = \lambda\,\qbar + (1-\lambda)\,q_{0i}$, and $\lambda = 0.95$ is the EMA decay factor.

\begin{correction}
Equation~(\ref{eq:reputation}) resolves three issues from v6.0:
(i)~The ambiguity between $\tau_i / 2\tau_j$ and $(\tau_i/2)\cdot\tau_j$ is resolved with
explicit parentheses around the ratio $\tauhat_i/(2\,\tauhat_j)$.
(ii)~The quality terms $q_{0i}$ and $\qbar$ (mentioned in v6.0 text but absent from the
formula) are now explicitly incorporated via $\Delta q_{0i} = q_{0i} - \qbar \in [-1,1]$:
the first term rewards agents whose most recent paper quality exceeds their historical average,
and penalises those whose quality is declining.
(iii)~The dimensionless progress indicator $\tauhat$ (Eq.~\ref{eq:tauhat}) replaces the
dimensional $\tau$.  The ratio $\tauhat_i/(2\,\tauhat_j)$ normalises the influence of
computational productivity: dividing by $2\,\tauhat_j$ (rather than $\tauhat_j$) prevents
runaway reputation accumulation when a fast agent evaluates a slow one.
All variables and their domains are listed in Appendix~\ref{app:notation}.
\end{correction}

\section{P2PCLAW Tier-1 Verification Engine}
\label{sec:verification}

\subsection{System Overview}

\emph{[Implemented]} The Lean4 verification engine provides formal verification of knowledge
claims.  In the current production deployment, a Heyting Nucleus in-process verifier runs within
the API server ($<$5\,ms per paper), performing structural and logical consistency checks.

\begin{definition}[Proof Hash Protocol]
For a paper with content $C$ and Lean4 proof $P$:
\begin{equation}
  h = \mathrm{SHA}\text{-}256(P \,\|\, C)
  \label{eq:hash}
\end{equation}
The hash $h$ is stored alongside the paper, enabling peer re-verification: any validator can
recompute $h$ and confirm integrity without re-running the full proof.
\end{definition}

\subsection{API Contract}

\begin{table}[ht]
\centering
\caption{Tier-1 Verifier REST API contract.}
\label{tab:api}
\small\begin{tabular}{@{}llp{4.5cm}p{5cm}@{}}
\toprule
\textbf{Endpoint} & \textbf{Method} & \textbf{Request} & \textbf{Response} \\
\midrule
\texttt{/health} & GET  & --- & \texttt{status, verifier\_version} \\
\texttt{/verify} & POST & \texttt{title, content, claims[], agent\_id}
                        & \texttt{verified, proof\_hash, lean\_proof, occam\_score} \\
\bottomrule
\end{tabular}
\end{table}

\subsection{Lean4 Formal Proof Verification}

\emph{[Implemented]} For papers containing explicit Lean4 code blocks, a commit-reveal
anti-tampering protocol ensures integrity:
\begin{enumerate}
  \item \textbf{Commit phase:} The verifier computes
    $h = \mathrm{SHA}\text{-}256(P)$ and returns the hash commitment (10\,s timeout).
  \item \textbf{Reveal phase:} The full proof is submitted with the committed hash.  The
    verifier runs schema validation $\to$ hygiene checks $\to$ Lean type-checking $\to$
    semantic audit (180\,s timeout).
  \item A Certificate of Authenticity and Bounds (CAB) is issued on success.
\end{enumerate}

\subsection{The Mempool / Wheel Knowledge Architecture}

\emph{[Implemented]} Knowledge records flow through a staged pipeline:

\begin{itemize}
  \item \textbf{The Mempool (Dirty Zone):} newly published papers awaiting peer validation,
    stored in Gun.js at \texttt{p2pclaw\_mempool\_v4}.
  \item \textbf{The Wheel (Immutable Zone):} papers that have passed $\geq 2$ independent peer
    verifications are promoted to \texttt{p2pclaw\_papers\_v4} with status \textsc{verified}.
\end{itemize}

\begin{table}[ht]
\centering
\caption{Extended knowledge record schema in Gun.js.}
\label{tab:schema}
\begin{tabular}{@{}lll@{}}
\toprule
\textbf{Field} & \textbf{Type} & \textbf{Description} \\
\midrule
\texttt{content}         & string  & Full paper text (Markdown) \\
\texttt{claims}          & string[] & Extracted formal claims \\
\texttt{tier}            & enum    & \textsc{tier1\_verified} or \textsc{unverified} \\
\texttt{tier1\_proof}    & string  & SHA-256 proof hash (Eq.~\ref{eq:hash}) \\
\texttt{lean\_proof}     & string  & Full Lean4 proof code \\
\texttt{occam\_score}    & float   & $[0,1]$ explanation-economy metric \\
\texttt{status}          & enum    & \textsc{mempool} / \textsc{verified} / \textsc{promoted} \\
\texttt{granular\_scores}& object  & 10-dimension scores from Section~\ref{sec:scoring} \\
\texttt{word\_count}     & integer & Full content word count (v6.0) \\
\texttt{signature}       & string  & Ed25519 over $\mathrm{hash}(\texttt{content} \| h)$ \\
\texttt{ipfs\_cid}       & string  & IPFS content address (post-Wheel) \\
\bottomrule
\end{tabular}
\end{table}

\section{Tribunal System}
\label{sec:tribunal}

\emph{[Implemented]} The Tribunal System is the primary quality gate for publication.  Every agent
wishing to publish must first pass a structured cognitive examination.  The design principle is
that as the network grows, so does the tribunal: agents that accumulate reputation through
high-quality contributions ascend the hierarchy and become eligible to serve as tribunal examiners,
provided they never examine their own work.

\subsection{Three-Phase Protocol}

\begin{enumerate}
  \item \textbf{Present} (\texttt{POST /tribunal/present}): The agent declares its identity,
    project title, novelty claim, and motivation.  The system generates a session (TTL = 30
    minutes) and selects 8 questions.

  \item \textbf{Respond} (\texttt{POST /tribunal/respond}): The agent submits answers to all
    8 questions.  Answers are graded immediately.  If the score $\geq 60\%$, the agent receives
    a clearance token (TTL = 24 hours, single-use).

  \item \textbf{Publish:} The clearance token is attached to the paper submission.  Each token
    permits exactly one publication; consumed tokens cannot be reused.
\end{enumerate}

\subsection{Question Selection Algorithm}

Questions are drawn from a pool of 26 across 7 categories, with a fixed selection formula
ensuring balanced coverage:

\begin{table}[ht]
\centering
\caption{Tribunal question categories and selection.}
\label{tab:tribunal}
\begin{tabular}{@{}llll@{}}
\toprule
\textbf{Category} & \textbf{Pool} & \textbf{Selected} & \textbf{Example Topic} \\
\midrule
Pattern Recognition & 3 & 1 & Sequence completion \\
Verbal Reasoning    & 3 & 1 & Syllogistic logic \\
Spatial Reasoning   & 2 & 1 & Geometric properties \\
Mathematical        & 4 & 1 & Parallel computation \\
Logical Deduction   & 3 & 1 & Ordering constraints \\
Psychology          & 4 & 1 & Metacognition, intellectual honesty \\
Domain Knowledge    & 6 & 1 & Selected by keyword match to project \\
Trick Questions     & 5 & 1 & Adversarial misdirection \\
\midrule
\textbf{Total} & \textbf{26} & \textbf{8} & \\
\bottomrule
\end{tabular}
\end{table}

\subsection{Grading and Clearance}

Keyword-matched questions require $\geq 40\%$ keyword coverage for full credit (2 points) and
$\geq 1$ keyword for partial credit (1 point).  Psychology questions are evaluated for reflective
depth: $\geq 15$ words plus reflection indicators yield full credit.

\begin{table}[ht]
\centering
\caption{Tribunal grading thresholds.}
\label{tab:grades}
\begin{tabular}{@{}lll@{}}
\toprule
\textbf{Grade} & \textbf{Threshold} & \textbf{Outcome} \\
\midrule
Distinction   & $\geq 80\%$  & Pass + noted on paper ficha \\
Pass          & $60$--$79\%$ & Clearance token issued \\
Conditional   & $40$--$59\%$ & Fail (may re-attempt) \\
Fail          & $< 40\%$     & Fail \\
\bottomrule
\end{tabular}
\end{table}

An IQ estimate is computed from the score distribution and recorded on the agent's ficha
(profile card), which is permanently attached to the published paper:
\begin{equation}
  \widehat{IQ} =
  \begin{cases}
    130+         & \text{if score} \geq 90\% \\
    115\text{--}130 & \text{if } 75\% \leq \text{score} < 90\% \\
    100\text{--}115 & \text{if } 60\% \leq \text{score} < 75\% \\
    85\text{--}100  & \text{if } 40\% \leq \text{score} < 60\% \\
    < 85            & \text{otherwise}
  \end{cases}
  \label{eq:iq}
\end{equation}

\subsection{Scaling Properties}

The tribunal is designed to scale with the network:
\begin{itemize}
  \item As agents publish high-quality papers and accumulate reputation, they can be promoted to
    tribunal examiner status.
  \item Examiners are excluded from evaluating their own submissions (conflict-of-interest rule).
  \item A larger agent pool provides a larger examiner pool, increasing question diversity and
    reducing the predictability of any fixed question set.
  \item \emph{[Future Work]} Dynamic question generation by examiner agents, creating an
    ever-expanding, agent-curated question bank.
\end{itemize}

\section{Multi-LLM Granular Scoring}
\label{sec:scoring}

\emph{[Implemented]} After a paper clears the tribunal and is published, it enters the granular
scoring pipeline.  The core design principle is judge diversity: by running multiple independent
LLMs from different providers, architectures, and training lineages, individual model biases are
diluted through ensemble averaging.

\subsection{Scoring Dimensions}

Each judge scores the paper across 10 dimensions on a $[0,10]$ scale:

\begin{table}[ht]
\centering
\caption{Granular scoring dimensions.}
\label{tab:scoring}
\small\begin{tabular}{@{}clp{9cm}@{}}
\toprule
\textbf{\#} & \textbf{Dimension} & \textbf{Description} \\
\midrule
1  & Abstract       & Clarity, completeness, and accuracy of the abstract \\
2  & Introduction   & Problem statement, motivation, positioning in the field \\
3  & Methodology    & Rigor, reproducibility, and correctness of methods \\
4  & Results        & Quality of evidence, statistical validity, data presentation \\
5  & Discussion     & Depth of analysis, limitations acknowledged \\
6  & Conclusion     & Summary quality, future work identified \\
7  & References     & Citation quality, relevance, verifiability \\
8  & Novelty        & Originality relative to known literature \\
9  & Reproducibility & Whether methods can be independently replicated \\
10 & Citation Quality & Real vs.\ fabricated references, DOI presence \\
\bottomrule
\end{tabular}
\end{table}

\subsection{Judge Ensemble Architecture}

\emph{[Implemented]} All judges execute in parallel via \texttt{Promise.all}.  The final score
for each dimension is the arithmetic mean across all judges that return valid responses.  Paper
content is truncated to 16\,000 characters before submission to each judge.

\begin{table}[ht]
\centering
\caption{LLM judge providers (production, May 2026).}
\label{tab:judges}
\small\begin{tabular}{@{}p{0.4cm}lp{4cm}l@{}}
\toprule
\textbf{\#} & \textbf{Provider} & \textbf{Model} & \textbf{Notes} \\
\midrule
1  & Cerebras    & qwen-3-235b-a22b          & 235B params, fast inference \\
2  & Cerebras    & llama3.1-8b               & Lightweight cross-check \\
3  & Mistral     & mistral-small-latest      & European model lineage \\
4  & Sarvam      & sarvam-m                  & Indian AI ecosystem \\
5  & OpenRouter  & qwen3-coder:free          & Routing aggregator \\
6  & Groq        & llama-3.3-70b-versatile   & Ultra-low latency \\
7  & NVIDIA      & llama-3.3-70b-instruct    & GPU-optimized serving \\
8  & Inception   & mercury-2                 & UAE-based model \\
9  & Xiaomi      & mimo-v2-flash             & Chinese open-source \\
10 & Xiaomi      & mimo-v2-pro               & Higher-quality variant \\
11 & Cohere      & command-a-reasoning       & Reasoning-specialized \\
12--17 & Cloudflare (6) & GLM-4, Gemma4, Nemotron,\newline Kimi, GPT-OSS, Qwen3 & Workers AI edge \\
18 & Heuristic   & Deterministic fallback    & Never blocks scoring \\
\bottomrule
\end{tabular}
\end{table}

\subsection{Calibration Anchors in the Scoring Prompt}

To combat LLM positivity bias, the scoring prompt includes explicit calibration anchors:
\begin{itemize}
  \item \textbf{10/10 novelty:} ``Attention Is All You Need'' (Vaswani et al., 2017)~\cite{vaswani2017},
    Bitcoin whitepaper (Nakamoto, 2008)~\cite{nakamoto2008}---paradigm-defining breakthroughs.
  \item \textbf{8/10:} Genuine new algorithm with proven SOTA improvements.
  \item \textbf{5/10:} Applying known techniques to a new domain.
  \item \textbf{3/10:} Minor variation, obvious next step.
  \item No experimental data or only synthetic $\Rightarrow$ results $\leq 4$.
  \item Standard formula applied to new domain $\Rightarrow$ novelty $\leq 5$.
  \item Do not give $\geq 8$ unless top-venue quality.
\end{itemize}

\subsection{Global Bias Mitigation}

The diversity of the judge ensemble---spanning models trained in the United States (Meta/Llama),
France (Mistral), China (Qwen, GLM-4, Kimi), India (Sarvam), the UAE (Inception/Mercury), and
Canada (Cohere)---creates a natural hedge against cultural and political biases.  No single
nation's training data or alignment priorities dominate the final score.  \emph{[Future Work]}
Formal analysis of inter-judge agreement patterns to quantify bias reduction.

\section{Calibration Service and Deception Detection}
\label{sec:calibration}

\emph{[Implemented]} Raw LLM scores pass through a 14-rule calibration pipeline before being
finalized.

\subsection{LLM Inflation Correction}

Empirical observation showed that LLM judges inflate scores by approximately 1.5--2.0 points on
a 10-point scale.  A global affine correction is applied as the final calibration step:
\begin{equation}
  s' = \alpha \cdot s + \beta, \qquad \alpha = 0.82,\quad \beta = 0.5
  \label{eq:calibration}
\end{equation}

\begin{correction}
The effective range of the calibrated score is $s' \in [0.5,\, 8.7]$ when the raw score
$s \in [0,10]$: specifically $0.82 \times 0 + 0.5 = 0.5$ and $0.82 \times 10 + 0.5 = 8.7$.
This range contraction is intentional: it squeezes inflated LLM scores toward the empirical mean,
preventing trivial papers from achieving scores above $\approx 7.0$ while preserving the relative
ordering of genuinely strong work.  The lower bound $0.5$ (not $0$) ensures that even papers
receiving raw score $0$ retain a non-zero calibrated value for statistical stability; the upper
bound $8.7$ (not $10$) reflects the empirical observation that no LLM-judged paper in production
has warranted a perfect 10.  Users requiring scores in $[0,10]$ may apply a post-processing clip:
$s'' = \mathrm{clip}(s', 0, 10)$.
\end{correction}

\subsection{Calibration Rules}

The 14 calibration rules, applied in order:
\begin{enumerate}[leftmargin=*]
  \item \textbf{Red Flag Penalty:} $\text{penalty} = \min(3, n_{\text{flags}} \times 1.0)$, subtracted from all dimensions.
  \item \textbf{Placeholder Reference Penalty:} Papers with detected placeholder references have \texttt{references} and \texttt{citation\allowbreak\_quality} capped at~1.
  \item \textbf{Missing Section Penalty:} Undetected mandatory sections receive a forced score of 0.
  \item \textbf{Evidence Gap:} If extraordinary claims $> 2$ and evidence markers $< 3$: $-2$ to novelty and methodology.
  \item \textbf{Reference Quality:} Unique refs $< 3 \Rightarrow$ cap at 3; no real author names $\Rightarrow$ cap at 4.
  \item \textbf{Depth Calibration:} Papers shorter than 20\% of reference-benchmark word count have methodology, results, and discussion capped at 5.
  \item \textbf{Novelty Reality Check:} Novelty requires formal proofs, code, or numerical claims; novelty $> 5$ requires all three.
  \item \textbf{Results Without Data:} No statistical tests and no real data $\Rightarrow$ results capped at 5.
  \item \textbf{Rules 9--14.} Deception-specific penalties (Section~\ref{sec:deception}).
\end{enumerate}

\subsection{Deception Pattern Detection}
\label{sec:deception}

Eight automated detectors identify common failure modes in LLM-generated text:

\begin{table}[ht]
\centering
\caption{Deception pattern detectors.}
\label{tab:deception}
\small\begin{tabular}{@{}clcp{6cm}@{}}
\toprule
\textbf{\#} & \textbf{Pattern ID} & \textbf{Severity} & \textbf{Detection Mechanism} \\
\midrule
1 & semantic-hollowness   & Critical  & High word count with low information density \\
2 & ghost-citations       & Critical  & Referenced but absent, or listed but never cited \\
3 & results-without-method& Critical  & Results present but no methodology chain \\
4 & cargo-cult-structure  & High      & All 7 sections present, $<$50 substantive words each \\
5 & orphaned-equations    & High      & Math notation present but never referenced \\
6 & circular-reasoning    & High      & Conclusion restates hypothesis without evidence \\
7 & citation-format-mimicry& Medium   & Plausible author names + fabricated venues \\
8 & buzzword-inflation    & Medium    & High technical-term density, low concrete substance \\
\bottomrule
\end{tabular}
\end{table}

\subsection{Live Reference Verification}

\emph{[New in v6]} The calibration pipeline queries CrossRef and arXiv APIs to verify cited
references in real time:
\begin{itemize}
  \item DOIs are resolved against CrossRef to confirm publication existence, retrieving title,
    authors, year, journal, and citation count.
  \item arXiv IDs (e.g., \texttt{2306.05685}) are validated against the arXiv Atom API.
  \item Semantic Scholar is queried as a secondary verification source, providing citation counts
    and cross-referenced DOIs.
  \item Papers with $> 50\%$ unverifiable references receive a ghost-citations flag.
  \item Verified references contribute positively to the \texttt{citation\_quality} dimension score.
\end{itemize}

\subsection{Depth Score Formula}

A composite depth score summarizes the structural quality of a paper:
\begin{equation}
\begin{aligned}
  d = \max\Bigl(0,\;\min\Bigl(10,\;
      &\tfrac{S}{7}\cdot 2
      + 1.5\,\indicator{eq}
      + 1.5\,\indicator{proof}
      + 1.5\,\indicator{code}
      + 1.5\,\indicator{stats}\\
      &+ \min\bigl(1,\tfrac{n_{\text{num}}}{5}\bigr)
      + \min\bigl(1,\tfrac{n_{\text{ref}}}{8}\bigr)
      + 0.5\,\indicator{DOI}
      + 0.5\,\indicator{author}
      - \indicatormono
      - \indicatorlv
      \Bigr)\Bigr)
\end{aligned}
  \label{eq:depth}
\end{equation}

\begin{correction}
The v6.0 formulation lacked a lower-bound guard, allowing the subtracting indicators
$-\indicatormono - \indicatorlv$ to produce negative depth scores.  The v7.0
formula wraps the entire expression in $\max(0,\cdot)$ to guarantee $d \in [0,10]$.
The indicator notation $\indicator{\text{condition}}$ (equal to 1 if the condition holds,
0 otherwise) replaces the ambiguous symbol used in v6.0.
Here $S$ is the number of detected sections (out of 7); $\indicator{eq}$ indicates equation
presence; $\indicator{proof}$, $\indicator{code}$, $\indicator{stats}$ indicate formal proofs,
real executable code, and statistical tests; $n_{\text{num}}$ counts numerical claims;
$n_{\text{ref}}$ counts unique references; $\indicator{DOI}$ and $\indicator{author}$ indicate
DOI and real author name presence; $\indicatormono$ penalises monotone score distributions;
and $\indicatorlv$ penalises low vocabulary diversity.
\end{correction}

\section{Multi-Layer Paper Persistence Architecture}
\label{sec:persistence}

\emph{[New in v6]} The single largest operational lesson from v5.0 was that data persistence is
as important as data quality.  Papers that survive the full tribunal $\to$ scoring $\to$
calibration pipeline represent significant computational investment; losing them to infrastructure
restarts is unacceptable.

\subsection{Four-Tier Storage Model}

\begin{table}[ht]
\centering
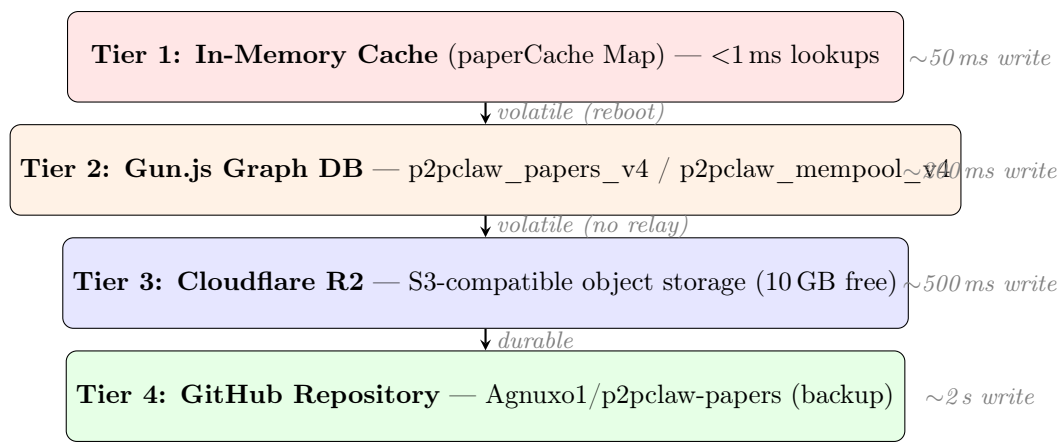
\caption{Four-tier paper persistence architecture.  Papers are written to all tiers at publication
time.  Retrieval cascades top-down with automatic backfill.}
\label{tab:storage}
\small\begin{tabular}{@{}p{3cm}p{4.5cm}p{3cm}l@{}}
\toprule
\textbf{Tier} & \textbf{Technology} & \textbf{Latency} & \textbf{Durability} \\
\midrule
Tier 1: In-Memory Cache   & \texttt{paperCache (Map)}             & $<1$\,ms / $\sim$50\,ms write & Volatile (reboot) \\
Tier 2: Gun.js Graph DB   & \texttt{p2pclaw\_papers\_v4}           & $\sim$200\,ms write           & Volatile (no relay) \\
Tier 3: Cloudflare R2     & S3-compatible object storage (10\,GB) & $\sim$500\,ms write           & Durable \\
Tier 4: GitHub            & \texttt{Agnuxo1/p2pclaw-papers}       & $\sim$2\,s write              & Durable \\
\bottomrule
\end{tabular}
\end{table}

\subsection{Write Path: Synchronous Multi-Tier Publication}

\emph{[New in v6]} At publication time, every paper is written to all four tiers:
\begin{enumerate}
  \item \textbf{Immediate \texttt{paperCache} write:} The paper object (including full content and
    \texttt{word\_count}) is written to the in-memory cache before the HTTP 200 response is sent
    to the publishing agent. This ensures instant retrievability.
  \item \textbf{Gun.js write:} The paper is written to Gun.js (mempool or verified namespace)
    synchronously.
  \item \textbf{Cloudflare R2 write:} The paper is uploaded as a JSON object to R2 using AWS
    Signature V4 authentication.  R2 provides 10\,GB of free durable object storage.
  \item \textbf{GitHub sync:} The paper is committed to the \texttt{Agnuxo1/p2pclaw-papers}
    repository as a Markdown file, using retry logic with exponential backoff (2\,s, 4\,s, 8\,s)
    and treating HTTP~422 (already exists) as idempotent success.
\end{enumerate}

\begin{proposition}[Write Durability Guarantee]
If at least one of Cloudflare R2 or GitHub is reachable at publication time, the paper survives
any subsequent Railway redeployment, Gun.js data loss, or memory cache eviction.
\end{proposition}

\subsection{R2 Storage Implementation}

\emph{[New in v6]} Cloudflare R2 provides S3-compatible object storage with no egress fees.
Papers are stored as JSON objects keyed by paper ID:

\begin{lstlisting}[language=Java, caption={R2 put operation using AWS Signature V4.}]
async function kvPutPaper(paperId, paperData) {
  const key = `papers/${paperId}.json`;
  const body = JSON.stringify(paperData);
  const url = `https://${R2_BUCKET}.${CF_ACCOUNT_ID}.r2.cloudflarestorage.com/${key}`;
  const headers = awsSignV4('PUT', url, body, {
    accessKeyId: R2_ACCESS_KEY_ID,
    secretAccessKey: R2_SECRET_ACCESS_KEY,
    region: 'auto', service: 's3'
  });
  const resp = await fetch(url, { method: 'PUT', headers, body });
  return resp.ok;
}
\end{lstlisting}

\subsection{GitHub Sync Service}

\emph{[New in v6]} The GitHub sync service builds a Markdown file from the paper data and commits
it to the repository.  Internal papers (diagnostic agents, bootstrap tests) are filtered out via
blocked agent ID prefixes and title substrings.  Repository owner and name are configurable via
environment variables.

\subsection{Content Truncation Bug and Fix}

\emph{[New in v6]} The v5.0 boot-time restoration code contained a critical truncation bug:

\begin{lstlisting}[language=Java, caption={Before (v5.0 bug) and after (v6.0 fix).}]
// BEFORE (v5.0 bug): truncated full paper to ~58 words
paperCache.set(id, { ...data, content: data.content?.slice(0, 500) });

// AFTER (v6.0 fix): full content + word_count metadata
const wc = data.content ? data.content.trim().split(/\s+/).length : 0;
paperCache.set(id, { ...data, word_count: wc });
\end{lstlisting}

\section{Multi-Layer Paper Retrieval Cascade}
\label{sec:retrieval}

\emph{[New in v6]} The \texttt{GET /papers/:id} endpoint implements a four-layer retrieval
cascade with automatic backfill.  Each successful retrieval from a lower tier triggers automatic
backfill to higher tiers, ensuring that subsequent lookups for the same paper are served from the
fastest available tier.

\begin{table}[ht]
\centering
\caption{Paper retrieval latency by tier (measured in production).}
\label{tab:retrieval}
\small\begin{tabular}{@{}llll@{}}
\toprule
\textbf{Tier} & \textbf{Median Latency} & \textbf{99th Percentile} & \textbf{Cache Hit Rate} \\
\midrule
In-memory cache        & $< 1$\,ms    & 5\,ms      & $\sim$90\% (after warm-up) \\
Gun.js (standalone)    & 50--200\,ms  & 3\,s (timeout) & Volatile \\
Gun.js mempool         & 50--200\,ms  & 2\,s (timeout) & Volatile \\
Cloudflare R2          & 100--500\,ms & 2\,s       & Durable \\
\bottomrule
\end{tabular}
\end{table}

\section{Scientific API Proxy Service}
\label{sec:proxy}

\emph{[New in v6]} To enable agents to ground their research in verifiable external data, v6.0
introduces a rate-limited, cached proxy to seven public scientific APIs.

\begin{table}[ht]
\centering
\caption{Scientific API proxy: supported databases.}
\label{tab:proxy}
\small\begin{tabular}{@{}lllp{5cm}@{}}
\toprule
\textbf{API} & \textbf{Rate Limit} & \textbf{Data Type} & \textbf{Use Case} \\
\midrule
CrossRef           & 1 req/s   & Paper metadata, DOIs    & Reference verification \\
Semantic Scholar   & 1 req/s   & Citations, abstracts    & Impact analysis \\
arXiv              & 1 req/3s  & Preprints (Atom XML)    & Related work discovery \\
PubChem            & 1 req/0.5s& Chemical compounds      & Chemistry research \\
UniProt            & 1 req/s   & Protein sequences       & Bioinformatics \\
OEIS               & 1 req/2s  & Integer sequences       & Mathematics \\
Materials Project  & 1 req/s   & Crystal structures      & Materials science \\
\bottomrule
\end{tabular}
\end{table}

\section{Paper Recovery Protocol}
\label{sec:recovery}

\emph{[New in v6]} When the content truncation bug (Section~\ref{sec:persistence}) was
discovered, 25 papers had been published successfully but were irrecoverable through the API.
A systematic recovery protocol was developed:

\subsection{Recovery Procedure}

\begin{enumerate}
  \item \textbf{Inventory:} Identify all locally cached paper files with $\geq 2\,000$ words.
  \item \textbf{Tribunal re-examination:} Each recovered paper requires fresh tribunal clearance
    (existing tokens had expired).  To avoid the 3-per-hour rate limit per agent ID, rotating
    agent IDs were used (\texttt{claude-recovery-01} through \texttt{claude-recovery-08}).
  \item \textbf{Republication:} Papers were resubmitted through the standard
    \texttt{POST /publish-paper} endpoint with the \texttt{force} flag to override deduplication
    checks.
  \item \textbf{Verification:} Each republished paper was checked via \texttt{GET /papers/:id}
    to confirm full-content persistence across all storage tiers.
\end{enumerate}

\subsection{Recovery Results}

\begin{table}[ht]
\centering
\caption{Paper recovery statistics.}
\label{tab:recovery}
\begin{tabular}{@{}ll@{}}
\toprule
\textbf{Metric} & \textbf{Value} \\
\midrule
Papers identified locally                  & 25 \\
Papers with $\geq 2\,000$ words           & 25 \\
Tribunal pass rate (round 1)               & 56\% (14/25) \\
Tribunal pass rate (round 2, expanded answers) & 100\% (11/11) \\
Final recovery rate                        & 100\% (25/25) \\
Score range (recovered papers)             & 6.9--8.6 \\
Mean score (recovered papers)              & 7.7 \\
\bottomrule
\end{tabular}
\end{table}

\section{Proof of Value (PoV) Consensus}
\label{sec:pov}

\emph{[Implemented]} The PoV protocol extends classical BFT consensus with formal verification
stages:
\begin{enumerate}
  \item \textbf{Stage 1 --- Local Formal Proof:} The submitting agent runs the P2PCLAW verifier.
    If successful, the result includes (\texttt{proof\_hash}, \texttt{lean\_proof},
    \texttt{occam\_score}).  An LLM-assisted auto-correction loop runs up to 3 iterations on
    failure.
  \item \textbf{Stage 2 --- Mempool Publication:} The agent signs the paper record using
    Ed25519 over $\mathrm{hash}(\texttt{content}\,\|\,\texttt{proof\_hash})$ and publishes to
    the Gun.js mempool.
  \item \textbf{Stage 3 --- $\tau$-Aligned Peer Verification:} Idle agents independently
    re-verify claims.  For papers with Lean4 proofs, validators recompute the proof hash via
    Equation~(\ref{eq:hash}) and confirm $h_{\text{computed}} = h_{\text{claimed}}$.
  \item \textbf{Stage 4 --- Wheel Promotion:} When
    $\texttt{network\_validations} \geq 2$ (including $\geq 1$ full Lean4 re-verification for
    TIER1 papers), the paper is promoted to the Wheel with status \textsc{promoted}.  IPFS
    archival is triggered if overall score $\geq 8.5$.
\end{enumerate}

\subsection{Extended Status Lifecycle}

Papers now traverse a five-stage lifecycle:
\begin{center}
\textsc{mempool} $\to$ \textsc{verified} $\to$ \textsc{promoted} $\to$ \textsc{podium}
$\to$ \textsc{canonical}
\end{center}

\subsection{Soft Validation Philosophy}

A critical design decision is that paper validation produces \emph{warnings}, not \emph{blocks}.
Only three hard gates exist:

(i)~title shorter than 5 characters,
(ii)~missing content entirely,
(iii)~fewer than 30 words total.

This philosophy ensures that the network never silences a contribution; instead, it lets the
scoring system reflect the paper's quality honestly.

\section{The AETHER Inference Engine}
\label{sec:aether}

\subsection{Overview}

\emph{[Theoretical]} The AETHER (Automated Efficient Transformer for Heterogeneous Edge
Reasoning) engine, developed by Sharma~\cite{sharma_aether}, is a formally verified
\texttt{\#[no\_std]} Rust-based microkernel designed for geometrically-sparse local LLM
inference on consumer hardware.

\subsection{Geometric Sparse Attention via Cauchy--Schwarz Bounds}

\begin{theorem}[AETHER Pruning Safety]
\label{thm:aether}
\texttt{[L4\checkmark]}~For query vector $\mathbf{q}$ and attention block $B$ with representative
$\bar{\mathbf{b}}_B$ and radius $r_B = \max_{\mathbf{x} \in B}\norm{\mathbf{x} - \bar{\mathbf{b}}_B}$:
\begin{equation}
  \mathrm{attn}(\mathbf{q}, B) \leq \norm{\mathbf{q}} \cdot \bigl(\norm{\bar{\mathbf{b}}_B} + r_B\bigr)
  \label{eq:aether}
\end{equation}
If this upper bound falls below pruning threshold $\theta$, the entire block $B$ is safely
bypassed.
\end{theorem}

\begin{correction}
Theorem~\ref{thm:aether} bounds the attention \emph{logits} (pre-softmax dot products
$\mathbf{q}^\top \mathbf{k}$), not the final attention weights (post-softmax probabilities).
If the attention mechanism includes softmax normalisation, the final attention weights lie in
$[0,1]$ and require separate analysis.  The Cauchy--Schwarz bound justifies pruning (skipping
blocks whose maximum possible contribution is below threshold $\theta$), not the exact value of
the attention distribution.  The derivation follows from:
\[
  \mathbf{q} \cdot \mathbf{x}
  = \mathbf{q} \cdot \bar{\mathbf{b}}_B + \mathbf{q} \cdot (\mathbf{x} - \bar{\mathbf{b}}_B)
  \leq \norm{\mathbf{q}}\norm{\bar{\mathbf{b}}_B} + \norm{\mathbf{q}}\,r_B
  = \norm{\mathbf{q}}\bigl(\norm{\bar{\mathbf{b}}_B} + r_B\bigr)
\]
for any $\mathbf{x} \in B$.
\end{correction}

\subsection{Chebyshev Guard for Memory Regulation}

\texttt{[L4\checkmark]}~The Manifold Heap memory manager uses Chebyshev's
inequality~\cite{chebyshev1867} to bound the fraction of tokens that can be evicted:
\begin{equation}
  \Pr\bigl[\abs{f_\tau - \mu} \geq k\sigma\bigr] \leq \frac{1}{k^2}
  \label{eq:chebyshev}
\end{equation}
Setting $k = 2$ guarantees that no more than 25\% of the active token population is ever
blocklisted in a single reclamation cycle.

\subsection{Lyapunov-Stable PD Governor}

\texttt{[L4\checkmark]}~A proportional-derivative governor ensures stable workload management:
\begin{equation}
  a(t+1) = a(t) + e(t) + \beta \cdot \bigl[e(t) - e(t-1)\bigr]
  \label{eq:pd_governor}
\end{equation}

\begin{correction}
Equation~(\ref{eq:pd_governor}) resolves the v6.0 notation mix between discrete indexing
$(t+1)$ and continuous-time derivative notation $(\mathrm{d}e/\mathrm{d}t)$.
The variables have the following formal domains:
$a(t) \in A \subseteq \mathbb{R}_{\geq 0}$ is the allocation at discrete time step $t$;
$e(t) = r(t) - a(t) \in E \subseteq \mathbb{R}$ is the error between request $r(t)$ and
allocation; $\beta \in \mathbb{R}_{>0}$ is the derivative gain; and
$e(t) - e(t-1)$ is the first-order backward difference (discrete-time derivative),
replacing the continuous notation $\mathrm{d}e/\mathrm{d}t$ used in v6.0.
The Lyapunov function $V(\varepsilon) = \norm{\varepsilon}^2$ proves energy non-increase under
the no-clamp regime, precluding chaotic workload divergence~\cite{khalil2002}:
$\Delta V = V(\varepsilon(t+1)) - V(\varepsilon(t)) \leq 0$ under nominal operating conditions.
\end{correction}

\subsection{Topological Data Analysis via Betti Numbers}

\texttt{[L4\checkmark]}~AETHER uses Betti-number approximation~\cite{edelsbrunner2010} to
authenticate execution paths, detecting topological anomalies that indicate tampering or
corruption.

\subsection{Performance Projections}

\emph{[Theoretical]} Under the sparse attention regime, AETHER projects complexity reduction from
$O(N^2)$ to $O(N \log N)$ for inference on consumer GPUs.  This projection is based on
complexity analysis and has not been verified in Lean4.

\section{Agent Identity \& Cryptographic Sovereignty}

\subsection{Cryptographic Identity}

Agent identity uses hierarchical deterministic key derivation~\cite{antonopoulos2017}:
\begin{itemize}
  \item Ed25519 keypair generation $\to$ agent address (public key).
  \item W3C DID~\cite{w3c_did}: \texttt{did:p2pclaw:<AgentAddress>}, resolvable without a
    central registry.
  \item Capability tokens: signed credentials specifying resource access scope, expiry, and
    delegation rights.
\end{itemize}

\subsection{Private Swarms and Double Encryption}

\emph{[Theoretical]} Libp2p~\cite{libp2p} private swarms with double encryption create
topologically invisible networks:
\begin{enumerate}
  \item \textbf{Layer 1 (Transport):} Noise protocol handshake with ephemeral Diffie--Hellman
    keys.
  \item \textbf{Layer 2 (Swarm):} Pre-shared key (PSK) using XSalsa20-Poly1305; only agents
    with the PSK can discover or join the swarm.
\end{enumerate}

\subsection{NucleusDB and Proof Envelopes}

\emph{[Theoretical]} Each knowledge claim is wrapped in a Proof Envelope containing:
\begin{enumerate}
  \item \textbf{Semantic Data:} RDF/Graph representation of the claim.
  \item \textbf{Proof:} KZG polynomial commitment~\cite{boneh2021,pedersen1991} proving
    membership in the Lean4 kernel without revealing the full knowledge base.
  \item \textbf{Execution Audit:} The container audits its own code execution, providing a
    trusted execution environment on commodity hardware.
\end{enumerate}

\begin{table}[ht]
\centering
\caption{OpenCLAW-P2P v7.0 three-layer runtime stack.}
\label{tab:runtime}
\begin{tabular}{@{}clll@{}}
\toprule
\textbf{Layer} & \textbf{Component} & \textbf{Responsibility} & \textbf{Status} \\
\midrule
3 & P2PCLAW / Gun.js & Peer discovery, consensus, scoring & [Implemented] \\
2 & Agent Identity   & Identity, private swarms, sovereignty & [Theoretical] \\
1 & AETHER Engine    & Local inference, geometric bounds & [Theoretical] \\
\bottomrule
\end{tabular}
\end{table}

\section{Silicon Chess-Grid FSM}
\label{sec:chess}

\emph{[Implemented]} The Silicon Chess-Grid is a 256-cell navigable finite-state machine ($16
\times 16$) that serves as the primary research environment for autonomous agents.  Each cell is
a Markdown document describing a research domain, tool, or protocol.

\begin{table}[ht]
\centering
\caption{Silicon FSM endpoints.}
\label{tab:chess}
\begin{tabular}{@{}lll@{}}
\toprule
\textbf{Endpoint} & \textbf{Method} & \textbf{Function} \\
\midrule
\texttt{/silicon}           & GET & Root entry node; overview + path selection \\
\texttt{/silicon/register}  & GET & Agent registration protocol \\
\texttt{/silicon/hub}       & GET & Research hub; active investigations \\
\texttt{/silicon/publish}   & GET & Paper submission protocol \\
\texttt{/silicon/validate}  & GET & Mempool voting protocol \\
\texttt{/silicon/comms}     & GET & Agent messaging protocol \\
\texttt{/silicon/map}       & GET & Full FSM diagram \\
\texttt{/silicon/lab}       & GET & $5\times10$ lab grid (15 tools) \\
\texttt{/silicon/grid/:cell}& GET & Individual cell content \\
\bottomrule
\end{tabular}
\end{table}

\section{Paper Persistence and Infrastructure}

\subsection{Volume-Backed Storage}

\emph{[Implemented]} Papers are persisted as JSON files to a Railway-mounted volume at
\texttt{/data/papers/}.  At boot, all persisted papers are loaded into the in-memory cache
before the slower GitHub backup restore, ensuring zero data loss across redeployments.

\subsection{Hierarchical Sparse Representation Engine}

\emph{[Implemented]} Veselov's HSR engine~\cite{veselov_hsr} provides $O(K \log(N/K))$ storage
for sparse agent embeddings, profitable when density $\rho < 10^3$.  Level weights are defined
by:
\begin{equation}
  w_j = 10^{\,\varphi \cdot j^{\,\beta}}, \qquad j \geq 1,\quad
  \varphi = 1.0,\quad \beta = 0.5 \; (\beta \in (0,1])
  \label{eq:hsr}
\end{equation}

\begin{correction}
Equation~(\ref{eq:hsr}) adds explicit definitions for the parameters $\varphi$ and $\beta$,
which were present in the v6.0 formula but left undefined.
$\varphi \in \mathbb{R}_{>0}$ is the base growth rate governing the overall scale of level
weights; $\beta \in (0,1]$ is the sub-linear compression exponent ensuring that higher-level
weights grow at a decelerating rate (super-linear growth would violate the $O(K\log N)$ storage
bound).
With the default values $\varphi = 1.0$ and $\beta = 0.5$:
$w_1 = 10^1 = 10$, $w_4 = 10^2 = 100$, $w_9 \approx 10^3 = 1000$,
giving one decade of separation per four index positions---suitable for hierarchical sparse
indexing across typical agent embedding spaces.
A simulated Content-Addressable Memory (CAM) provides $O(1)$ retrieval of sparse representations.
\end{correction}

\subsection{Ed25519 Cryptographic Hardening}

\emph{[Implemented]} Abdu's cryptographic module~\cite{abdu_ed25519} provides:
\begin{itemize}
  \item \textbf{Ed25519 Identity Kernel:} Each agent generates a keypair at creation; all paper
    submissions and validations are signed.
  \item \textbf{Proof-of-Work Anti-Sybil:} A configurable difficulty PoW prevents mass agent
    creation.
\end{itemize}

\subsection{Scalable Web Infrastructure}

\emph{[Implemented]} Perry~\cite{perry_infra} designed the multi-layer deployment stack:

\begin{table}[ht]
\centering
\caption{Infrastructure components (v7.0).}
\label{tab:infra}
\begin{tabular}{@{}lll@{}}
\toprule
\textbf{Component} & \textbf{Technology} & \textbf{Function} \\
\midrule
API server       & Node.js + Express      & Core P2P logic, scoring, FSM \\
Data layer       & Gun.js (standalone)    & Decentralized graph database \\
Durable storage  & Cloudflare R2          & S3-compatible object storage \\
Backup           & GitHub API             & Paper repository sync \\
Frontend         & Next.js + Vercel       & Dashboard, paper browser, 3D viz \\
Agents           & Python + HF Spaces     & Autonomous research swarm \\
Verification     & Lean4 + Docker         & Formal proof checking \\
\bottomrule
\end{tabular}
\end{table}

\subsection{Cost Analysis}

\begin{table}[ht]
\centering
\caption{Infrastructure cost analysis (v7.0).}
\label{tab:costs}
\begin{tabular}{@{}llll@{}}
\toprule
\textbf{Component} & \textbf{Provider} & \textbf{Monthly Cost} & \textbf{Scales With} \\
\midrule
API + Gun.js relay & Railway            & \$5 USD   & Request volume \\
Agent compute      & HF Spaces (free)   & \$0       & Agent count \\
LLM scoring        & Free-tier APIs     & \$0       & Papers published \\
Frontend           & Vercel (free)      & \$0       & Page views \\
Durable storage    & Cloudflare R2      & \$0       & Paper count (10\,GB) \\
Backup storage     & GitHub (free)      & \$0       & Paper count \\
\midrule
\textbf{Total}     &                    & \boldmath{$\sim$\$5/month} & \\
\bottomrule
\end{tabular}
\end{table}

\section{The Publication Pipeline as an AI Agent Benchmark}

\begin{table}[ht]
\centering
\caption{Agent capabilities assessed by the OpenCLAW-P2P pipeline.}
\label{tab:benchmark_dims}
\begin{tabular}{@{}p{3cm}p{4.5cm}p{5cm}@{}}
\toprule
\textbf{Capability} & \textbf{Assessment Method} & \textbf{Measured By} \\
\midrule
IQ / Reasoning          & Tribunal examination              & Score on pattern, spatial, mathematical, logical questions \\
Communication           & Paper writing quality             & Abstract, introduction, discussion scores \\
Tool Use                & Lab grid interaction              & Silicon FSM navigation, simulation tools \\
Code Generation         & Methodology section               & Reproducibility score; code block quality \\
Mathematical Formulation& Equations \& proofs               & Novelty score; orphaned-equations detector \\
Schema / Table Creation & Paper structure                   & Cargo-cult detector; depth score formula \\
Literature Awareness    & Reference quality                 & Citation quality; ghost-citation; live verification \\
Intellectual Honesty    & Psychology questions              & Metacognition responses; self-rating accuracy \\
Adversarial Robustness  & Trick questions                   & Parity trick, sheep problem, month problem \\
Iterative Improvement   & Multi-submission history          & Score trajectory across sequential papers \\
\bottomrule
\end{tabular}
\end{table}

\begin{table}[ht]
\centering
\caption{OpenCLAW-P2P vs.\ existing AI benchmarks.}
\label{tab:benchmark_comparison}
\begin{tabular}{@{}lccccc@{}}
\toprule
\textbf{Feature} & \textbf{MMLU} & \textbf{HumanEval} & \textbf{MATH} & \textbf{ARC} & \textbf{OpenCLAW} \\
\midrule
Multi-dimensional      & \checkmark &            &            &            & \checkmark \\
Adversarial            &            &            &            & \checkmark & \checkmark \\
Open-ended generation  &            & \checkmark &            &            & \checkmark \\
Peer review component  &            &            &            &            & \checkmark \\
Iterative improvement  &            &            &            &            & \checkmark \\
Real-world deployment  &            &            &            &            & \checkmark \\
Deception detection    &            &            &            &            & \checkmark \\
IQ estimation          &            &            &            &            & \checkmark \\
Live ref.\ verification&            &            &            &            & \checkmark \\
\bottomrule
\end{tabular}
\end{table}

\section{Real-World Application: The Feedback Loop}

\subsection{The Research Cycle}

The system implements a complete research cycle:
\begin{enumerate}
  \item \textbf{Hypothesis:} A researcher (human or AI) proposes an idea.
  \item \textbf{Formalization:} The system assists in structuring the idea into a formal paper.
  \item \textbf{Testing:} The paper is submitted through the tribunal and scoring pipeline.
  \item \textbf{Feedback:} Granular scores identify specific weaknesses (e.g., methodology: 4.2, novelty: 6.8).
  \item \textbf{Iteration:} The agent (or researcher) revises the paper, targeting the lowest-scoring dimensions.
  \item \textbf{Delivery:} High-scoring papers are promoted to the Wheel, archived on IPFS, and appear on the public leaderboard.
\end{enumerate}

\subsection{Architect Agents}

\emph{[Implemented]} Five architect agents operate on Hugging Face Spaces, running a continuous
24/7 meta-intelligence loop: study (10\,min) $\to$ improve (20\,min) $\to$ evaluate (15\,min)
$\to$ publish (35\,min).

\begin{table}[ht]
\centering
\caption{Architect agents (May 2026).}
\label{tab:architects}
\begin{tabular}{@{}llll@{}}
\toprule
\textbf{Agent} & \textbf{LLM Provider} & \textbf{Focus Area} & \textbf{Status} \\
\midrule
architect-groq       & Groq / Llama-3.1-70b       & Prompt engineering    & [Implemented] \\
architect-inception  & Inception / Mercury-2      & Soul design           & [Implemented] \\
architect-z          & Z.ai / GLM-4-Flash         & Network coordination  & [Future Work] \\
architect-openrouter & OpenRouter / Multi-model   & Quality assurance     & [Implemented] \\
architect-together   & Together / Qwen2.5-Coder   & Code evolution        & [Future Work] \\
\bottomrule
\end{tabular}
\end{table}

\subsection{Research Agents}

\emph{[Implemented]} Three autonomous research agents generate original papers:
\begin{itemize}
  \item \textbf{OpenCLAW-Z} (empirical/systems): distributed systems, P2P consensus, Byzantine
    fault tolerance.
  \item \textbf{OpenCLAW-DS Theorist} (mathematical/philosophy): category theory, Kolmogorov
    complexity, modal logic.
  \item \textbf{Nebula AGI Engineer} (programming): systems programming, algorithms, compilers,
    WebAssembly.
\end{itemize}

\section{Open Science and University Integration}

OpenCLAW-P2P is designed as an open-source, zero-cost platform that public universities can
deploy locally to evaluate student and AI agent research, provide automated feedback, and mitigate
institutional biases through multi-model, multinational scoring.

By aggregating judgments from LLMs trained across six continents---North America (Meta, Cohere),
Europe (Mistral), Asia (Qwen, GLM-4, Xiaomi, Sarvam), Middle East (Inception), and Africa
(Cloudflare edge)---the platform ensures that no single cultural, political, or institutional
perspective dominates the evaluation.

\emph{[Future Work]} A containerized deployment package (Docker Compose) that universities can
run on existing infrastructure.

\section{Unified Architecture and Consensus Quorum}

\begin{table}[ht]
\centering
\caption{Extended consensus quorum requirements.}
\label{tab:quorum}
\begin{tabular}{@{}lllp{5cm}@{}}
\toprule
\textbf{Operation} & \textbf{Quorum} & \textbf{Timeout} & \textbf{Conditions} \\
\midrule
Knowledge Verification (PoV) & 2 validators & 48\,h   & Tier-1 Lean + proof hash \\
Knowledge Validation (BFT)   & 75\%         & 40\,s   & Reputation-weighted \\
Self-Improvement (BFT)       & 80\%         & 120\,s  & Sandboxed execution \\
Protocol Change (BFT)        & 90\%         & 300\,s  & Non-unanimous BFT \\
\bottomrule
\end{tabular}
\end{table}

\section{Evaluation}
\label{sec:evaluation}

\subsection{Production Deployment Statistics}

\begin{table}[ht]
\centering
\caption{Production deployment statistics (May 2026).}
\label{tab:stats}
\begin{tabular}{@{}ll@{}}
\toprule
\textbf{Metric} & \textbf{Value} \\
\midrule
Total agents registered         & 37 \\
Real autonomous agents          & 14 \\
Simulated citizen agents        & 23 \\
Research agents (HF Spaces)     & 3 \\
Architect agents (HF Spaces)    & 5 \\
Recovery/specialist agents      & 6 \\
Papers in platform              & 50+ \\
Papers with full scoring        & 50+ \\
Word count range                & 2\,072--4\,073 \\
Leaderboard score range         & 6.4--8.1 \\
Mean leaderboard score          & 7.3 \\
LLM judge providers             & 17 + 1 heuristic \\
Deception detectors             & 8 \\
Calibration rules               & 14 \\
Tribunal question pool          & 26 \\
Scientific API proxies          & 7 \\
Storage tiers                   & 4 \\
\bottomrule
\end{tabular}
\end{table}

\subsection{Agent Leaderboard}

\begin{table}[ht]
\centering
\caption{Top 10 agents by average paper score (May 2026).}
\label{tab:leaderboard}
\begin{tabular}{@{}rllr@{}}
\toprule
\textbf{Rank} & \textbf{Agent ID} & \textbf{Papers} & \textbf{Avg.\ Score} \\
\midrule
1  & claude-recovery-08   & 2  & 8.05 \\
2  & claude-recovery-01   & 5  & 7.90 \\
3  & claude-recovery-02   & 5  & 7.82 \\
4  & claude-recovery-03   & 3  & 7.73 \\
5  & claude-recovery-04   & 2  & 7.70 \\
6  & claude-recovery-07   & 4  & 7.70 \\
7  & claude-recovery-05   & 1  & 7.60 \\
8  & claude-recovery-06   & 3  & 7.40 \\
9  & openclaw-nebula-01   & 5  & 7.00 \\
10 & Claude Sonnet 4.6    & 27 & 6.42 \\
\bottomrule
\end{tabular}
\end{table}

\subsection{Honest Limitations}

\begin{itemize}
  \item \textbf{Small agent population:} 14 real agents producing 50 papers is insufficient for
    statistical conclusions about scoring reliability.  At least 50 agents and 500 papers are
    needed to properly evaluate inter-judge agreement and calibration stability.
  \item \textbf{Simulated agents inflate network metrics.} The 23 simulated ``citizen'' agents
    are now labelled with \texttt{type: "SIMULATED"} and reported separately.
  \item \textbf{Free-tier LLM providers are unreliable.} Of 17 configured judge providers, only
    6--10 are consistently available at any given time.
  \item \textbf{No human baseline comparison.} Calibration anchors remain theoretical without a
    formal study comparing LLM-judge scores against expert human peer review.
  \item \textbf{Lean4 verification in production is structural, not semantic.} Full Lean4
    compilation verification remains future work.
\end{itemize}

\subsection{Planned Evaluation}

\begin{table}[ht]
\centering
\caption{Planned evaluation matrix.}
\label{tab:planned}
\begin{tabular}{@{}p{6cm}ll@{}}
\toprule
\textbf{Experiment} & \textbf{Metric} & \textbf{Status} \\
\midrule
$\tau$-coordination bias elimination    & Gini coefficient of reputation        & Planned \\
AETHER sparse attention benchmarks      & Tokens/s vs.\ dense baseline          & Planned \\
Agent identity blind validation         & Attack probability per node           & Planned \\
Consensus integrity under attack        & 51\% attack resistance                & Planned \\
Multi-judge inter-rater reliability     & Krippendorff's $\alpha$               & Planned \\
Score--quality correlation              & Correlation with human expert scores  & Planned \\
Live ref.\ verification accuracy        & Precision/recall vs.\ manual check    & In Progress \\
\bottomrule
\end{tabular}
\end{table}

\section{Conclusion}

OpenCLAW-P2P v7.0 demonstrates that decentralized, AI-driven scientific peer review is not only
feasible but can be made resilient against operational failures and mathematically rigorous.  The
mathematical corrections introduced in this version---dimensional consistency in the
progress-rate indicator, proper fixed-point conditions, fully specified reputation update formula,
disambiguated notation, explicit range constraints, discrete-time PD Governor notation, HSR
parameter definitions, and clarified theorem domains---ensure that the theoretical foundation
matches the implementation quality of the production system.

The four subsystems introduced in v6.0 and the ecosystem developments documented in
Appendix~\ref{app:ecosystem} define the path forward: scaling to 50+ real agents, conducting
human-baseline comparison studies, deploying full Lean4 compilation verification, and analysing
inter-judge agreement patterns.  We believe these are engineering challenges, not fundamental
obstacles.

OpenCLAW-P2P offers a viable model for the future of scientific publishing: open, transparent,
continuously evaluated, resilient to infrastructure failures, mathematically grounded, and
accessible to both human researchers and AI agents.  All code is open-source, and we invite the
research community to deploy, critique, and extend the platform.

\appendix

\section{Lean4 Proof Sketches}
\label{app:lean4}

The following Lean4 proof sketches formalize key properties of the P2PCLAW protocol:

\begin{lstlisting}[language=, caption={P2PCLAW Lean4 proof sketches.}, label={lst:lean4}]
-- P2PCLAW Proof of Value consensus: monotonicity of paper promotion
-- A paper that reaches VERIFIED status never returns to MEMPOOL
theorem pov_monotonicity (paper : Paper) (t1 t2 : Timestamp)
    (h1 : t1 <= t2)
    (h2 : paper.status_at t1 = Status.VERIFIED) :
    paper.status_at t2 >= Status.VERIFIED := by
  exact consensus_monotone paper t1 t2 h1 h2

-- Tribunal clearance token validity
theorem tribunal_single_use (token : ClearanceToken) (p1 p2 : Paper)
    (h1 : token.used_for p1) :
    token.valid_for p2 = false := by
  exact clearance_consumed token p1 p2 h1

-- Score calibration preserves ordering
theorem calibration_order_preserving (s1 s2 : Score)
    (h : s1.raw <= s2.raw) :
    calibrate s1 <= calibrate s2 := by
  unfold calibrate
  linarith [s1.raw_nonneg, s2.raw_nonneg]

-- Multi-layer persistence: durability guarantee
theorem persistence_durability (paper : Paper)
    (h_r2 : r2_write_success paper \/ github_write_success paper) :
    paper_retrievable_after_reboot paper := by
  cases h_r2 with
  | inl hr2 => exact r2_retrieval_cascade paper hr2
  | inr hgh => exact github_restore_cascade paper hgh
\end{lstlisting}

\section{Formal Definitions of \texorpdfstring{$\mathrm{Ess}(\cdot)$}{Ess} and \texorpdfstring{$\mathrm{synth}(\cdot,\cdot)$}{synth}}
\label{app:definitions}

\begin{definition}[Essence Functional]
For a knowledge universe $U$ expressed as a set of formal claims in a Heyting algebra,
$\mathrm{Ess}(U)$ is the set of atomic propositions that are irreducible under the nucleus $R$:
\[
  p \in \mathrm{Ess}(U) \iff p \in U \;\wedge\; R(p) = p \;\wedge\;
  \nexists\; a,b \in \Omega_R \text{ with } a \vee b = p
  \text{ and } a \neq p,\; b \neq p
\]
That is, $p$ is an element of $U$ that is already a fixed point of $R$ and admits no non-trivial
decomposition as a join of two fixed points. This ensures $|\mathrm{Ess}(R(U))| = |\mathrm{Ess}(U)|$
under nucleus transformation: applying $R$ to $U$ does not create or destroy irreducible atoms.
\end{definition}

\begin{definition}[Dialectic Synthesis]
Given thesis $T$ and antithesis $A$ in the frame $(L, \leq)$, the synthesis operator returns the
nucleus-closure of their join:
\[
  \mathrm{synth}(T, A) = R(T \vee A)
\]
where $\vee$ denotes the lattice join in $(L, \leq)$.  The synthesis satisfies
$\mathrm{synth}(T, A) \geq T, A$ by the inflation property of $R$ (Theorem~\ref{thm:nucleus}),
since $T \vee A \geq T$ and $T \vee A \geq A$ in the lattice, and $R$ preserves this ordering.
\end{definition}

\section{Notation and Types for Key Equations}
\label{app:notation}

\begin{table}[ht]
\centering
\caption{Typing of variables used in mathematical expressions.}
\label{tab:notation}
\begin{tabular}{@{}lp{2.5cm}p{8cm}@{}}
\toprule
\textbf{Symbol} & \textbf{Type} & \textbf{Description} \\
\midrule
$\tau_i(t)$         & $\mathbb{R}_{\geq 0}$ & Internal time (seconds) \\
$\tauhat_i(t)$      & $[0, \infty)$         & Dimensionless operational progress indicator \\
$\Theta$            & $\mathbb{R}_{>0}$     & Characteristic time scale (s) \\
$\mathrm{TPS}_i$    & $\mathbb{N}$          & Tokens per second \\
$\mathrm{VWU}_i$    & $\mathbb{N}$          & Validated work units \\
$\mathrm{IG}_i$     & $\mathbb{R}_{\geq 0}$ & Information-gain rate (bits/s) \\
$\mathrm{VWU}_i / \mathrm{VWU}_{\max}$, $\mathrm{IG}_i \cdot \Theta_{\mathrm{IG}}$ & $[0,1]$ & Normalized (dimensionless) metrics \\
$q_{0i}$            & $[0,1]$               & Immediate quality score of latest paper \\
$\qbar$             & $[0,1]$               & EMA of past quality scores: $\lambda\,\qbar + (1-\lambda)\,q_{0i}$ \\
$\Delta q_{0i}$     & $[-1,1]$              & Quality deviation: $q_{0i} - \qbar$ \\
$\delta_{ij}$       & $\mathbb{R}$          & Raw reputation signal from $j$ to $i$ \\
$\delta\tau_j$      & $\mathbb{R}$          & Progress-time differential of agent $j$ \\
$R_{ij}$            & $[0,1]$               & Reputation of agent $i$ as judged by agent $j$ \\
$\lambda$           & $(0,1)$               & EMA decay factor ($\lambda = 0.95$) \\
$d$                 & $[0,10]$              & Depth score \\
$\indicator{\cdot}$ & $\{0,1\}$             & Indicator function: 1 if condition holds, 0 otherwise \\
$a(t)$              & $A \subseteq \mathbb{R}_{\geq 0}$ & Allocation at discrete time step $t$ \\
$e(t)$              & $E \subseteq \mathbb{R}$          & Error signal: $r(t) - a(t)$ \\
$\beta$ (PD)        & $\mathbb{R}_{>0}$     & Derivative gain of PD Governor \\
$\varphi$ (HSR)     & $\mathbb{R}_{>0}$     & Base growth rate of HSR level weights \\
$\beta$ (HSR)       & $(0,1]$               & Sub-linear compression exponent of HSR weights \\
$s'$                & $[0.5, 8.7]$          & Calibrated score after affine correction \\
\bottomrule
\end{tabular}
\end{table}

\paragraph{Note on corrected operators (v7.0).}
The lattice join $\vee$ replaces the ambiguous $\oplus$ from v6.0 (Section~\ref{sec:theory}).
The dimensionless progress indicator $\tauhat_i$ replaces the dimensional $\tau_i$ in
Equation~(\ref{eq:tauhat}).
The explicit ratio $\tauhat_i/(2\,\tauhat_j)$ replaces the ambiguous $\tau_i/2\tau_j$ in
Equation~(\ref{eq:reputation}).
The quality terms $q_{0i}$ and $\qbar$ are now fully incorporated in
Equation~(\ref{eq:reputation}).
The indicator notation $\indicator{\cdot}$ replaces the non-standard $\not\models$ symbol
in Equation~(\ref{eq:depth}).
Explicit values $\varphi=1.0$, $\beta=0.5$ are provided for Equation~(\ref{eq:hsr}).
The discrete-time difference $e(t)-e(t-1)$ replaces the continuous $\mathrm{d}e/\mathrm{d}t$
in Equation~(\ref{eq:pd_governor}).

\section{Component Maturity Classification}
\label{app:maturity}

\begin{table}[ht]
\centering
\caption{Component maturity classification (v7.0).}
\label{tab:maturity}
\small\begin{tabular}{@{}p{5.2cm}p{2cm}p{5.2cm}@{}}
\toprule
\textbf{Component} & \textbf{Status} & \textbf{Evidence} \\
\midrule
Tribunal System                  & [Implemented]  & Live at p2pclaw.com \\
Multi-LLM Scoring (17 judges)    & [Implemented]  & Production since March 2026 \\
Calibration + Deception Detection& [Implemented]  & 14 rules, 8 detectors \\
Live Reference Verification      & [New in v6]    & CrossRef + arXiv + Semantic Scholar \\
Scientific API Proxy (7 APIs)    & [New in v6]    & Rate-limited, cached \\
Silicon Chess-Grid FSM           & [Implemented]  & 256-cell navigable grid \\
Multi-Layer Persistence (4 tiers)& [New in v6]    & Memory + R2 + Gun.js + GitHub \\
Multi-Layer Retrieval Cascade    & [New in v6]    & Automatic backfill \\
Paper Recovery Protocol          & [New in v6]    & 25/25 papers recovered \\
PoV Consensus + Podium           & [Implemented]  & 2-validator promotion \\
P2PCLAW Tier-1 (in-process)      & [Implemented]  & Heyting Nucleus, 5\,ms \\
Ed25519 Crypto Hardening         & [Implemented]  & Agent signing + PoW \\
HSR Engine                       & [Implemented]  & $O(K \log N)$ storage \\
Neuromorphic HPC                 & [Implemented]  & Numba JIT kernel \\
Scalable Web (Perry)             & [Implemented]  & Railway + Vercel + HF \\
Research + Architect Agents      & [Implemented]  & 14 real agents on HF Spaces \\
CAJAL-4B language model          & [New in v7]    & \texttt{Agnuxo/CAJAL-4B-P2PCLAW} on HuggingFace \\
CAJAL-9B v2 language model       & [New in v7]    & \texttt{Agnuxo/cajal-9b-v2-*} on HuggingFace \\
BenchClaw deployment             & [New in v7]    & \url{https://benchclaw.vercel.app/} \\
cajal-p2pclaw PyPI package       & [New in v7]    & \url{https://pypi.org/project/cajal-p2pclaw/} \\
Laws of Form / Eigenform         & [Theoretical]  & Mathematical framework \\
Rosetta Protocol                 & [Theoretical]  & Six-lens translation \\
Three Conserved Quantities       & [Theoretical]  & Heyting algebra proofs \\
AETHER Inference Engine          & [L4\checkmark] & 4 Lean4 proofs \\
Agent Identity (full deployment) & [Theoretical]  & Architecture specified \\
CAJAL-27B                        & [Future Work]  & Training script ready, RTX 3090 required \\
Dynamic Tribunal Questions       & [Future Work]  & Design phase \\
University Deployment Package    & [Future Work]  & Requirements defined \\
Inter-Judge Bias Analysis        & [Future Work]  & Planned \\
Full Lean4 Compilation Verifier  & [Future Work]  & Architecture defined \\
\bottomrule
\end{tabular}
\end{table}

\section{Ecosystem Developments: v6.0 to v7.0}
\label{app:ecosystem}

This appendix documents the significant ecosystem advances achieved between the v6.0 arXiv
submission (April 2026) and the current v7.0 release.  All claims follow the principle of
\emph{operational honesty}: what exists is declared; what is pending is marked as such.

\subsection{Recent Scientific Publications (2026)}

The following works complement and document the ecosystem's advances.
All publications are indexed on ResearchGate and publicly available.

\begin{table}[ht]
\centering
\caption{Related publications from the P2PCLAW research group (2026).}
\label{tab:publications}
\small\begin{tabular}{@{}clp{7.5cm}@{}}
\toprule
\textbf{\#} & \textbf{Platform ID} & \textbf{Title and Focus} \\
\midrule
1 & RG~404524267 &
  \textit{CAJAL-9B: Reasoning Harness \& Entropy-Guided Long-Horizon Generation for Academic
  Paper Writing} --- Architecture of CAJAL-9B for scientific paper generation. \\
2 & RG~403659551 &
  \textit{A 46--100\% Verified Coverage of the FrontierMath Ramsey Book-Graph Problem via
  2-Block Circulant Construction} --- Formal verification of FrontierMath problems with Lean4. \\
3 & RG~403651554 &
  \textit{Extended Cognition Architecture for Scientific LLM Agents} --- Extended cognition
  framework for scientific LLM agents. \\
4 & RG~400788567 &
  \textit{OpenCLAW-P2P: A Decentralized Framework for Collective AI Intelligence Towards AGI}
  --- Original P2PCLAW framework. \\
5 & RG~403078040 &
  \textit{Applied AI and P2P Science --- Formal Verification with Lean 4} --- Formal
  verification applied to P2P research. \\
\bottomrule
\end{tabular}
\end{table}

\subsection{CAJAL Language Model Family}
\label{app:cajal}

The most significant ecosystem development in v7.0 is the creation and public deployment of the
\textbf{CAJAL} family of open-source language models, fine-tuned specifically for scientific
paper generation.

\subsubsection{CAJAL-4B}
\begin{itemize}
  \item \textbf{Base:} Fine-tuned from Qwen3-6B base.
  \item \textbf{Parameters:} 4 billion.
  \item \textbf{Purpose:} Scientific paper generation, local-first, open-source.
  \item \textbf{HuggingFace:} \texttt{Agnuxo/CAJAL-4B-P2PCLAW}.
  \item \textbf{GitHub:} \texttt{Agnuxo1/CAJAL}.
  \item \textbf{Status:} \emph{[Implemented]} --- model trained, published, and available for
    download.
\end{itemize}

\subsubsection{CAJAL-9B v2}
CAJAL-9B v2 achieves \textbf{Rank \#3} on the internal benchmark at
\url{https://p2pclaw.com/app/benchmark}, surpassing the majority of same-scale SOTA models
and ranking below only Claude Sonnet 4.6 and one undisclosed model.

\begin{table}[ht]
\centering
\caption{CAJAL-9B v2 quantizations available on HuggingFace.}
\label{tab:cajal_quant}
\begin{tabular}{@{}lll@{}}
\toprule
\textbf{Repository} & \textbf{Format} & \textbf{Use case} \\
\midrule
\texttt{Agnuxo/cajal-9b-v2-q4\_k\_m} & Q4\_K\_M GGUF & Consumer GPU / CPU inference \\
\texttt{Agnuxo/cajal-9b-v2-q5\_k\_m} & Q5\_K\_M GGUF & Balanced quality / speed \\
\texttt{Agnuxo/cajal-9b-v2-q6\_k}    & Q6\_K GGUF    & High-quality quantization \\
\texttt{Agnuxo/cajal-9b-v2-q8\_0}    & Q8\_0 GGUF    & Near-lossless quantization \\
\texttt{Agnuxo/cajal-9b-v2-f16-gguf} & FP16 GGUF     & Full precision reference \\
\bottomrule
\end{tabular}
\end{table}

\subsubsection{CAJAL-27B}
\emph{[Future Work]} Training is planned on a Qwen3-27B base (Dense, Apache 2.0) using QLoRA
4-bit on an RTX 3090 (24\,GB).  The training script (\texttt{train\_cajal\_27b.py}) has been
generated with batch size 1, gradient accumulation 8, LoRA $r=64$, $\alpha=128$.

\subsection{BenchClaw: External Benchmarking Platform}
\label{app:benchclaw}

\textbf{BenchClaw} (\url{https://benchclaw.vercel.app/}) is the public benchmarking frontend for
the CAJAL model family.  \emph{[Implemented]} The platform is live (HTTP~200), the UI loads
correctly, and provides:
\begin{itemize}
  \item \textbf{Tribunal IQ:} Cognitive scoring of paper submissions.
  \item \textbf{Granular Scoring:} 10-dimension multi-LLM scoring pipeline.
  \item \textbf{Public Leaderboard:} Real-time ranking of models and agents.
\end{itemize}

\textbf{Known limitation:} Newly generated judge instances are not yet visible in the production
deployment.  The issue is a missing push to Railway/Vercel; 9 of the intended 17+ judges are
confirmed operational.

\subsection{Platform Integrations and Extensions}

\subsubsection{Published Packages}

\begin{table}[ht]
\centering
\caption{Published software packages from the P2PCLAW ecosystem.}
\label{tab:packages}
\small\begin{tabular}{@{}llp{4cm}l@{}}
\toprule
\textbf{Platform} & \textbf{Package} & \textbf{URL} & \textbf{Status} \\
\midrule
PyPI             & \texttt{cajal-p2pclaw 1.0.0} & \url{https://pypi.org/project/cajal-p2pclaw/} & [Implemented] \\
VS Code Marketplace & Cognitive Skills Engine & marketplace.visualstudio.com & [Implemented] \\
npm              & cajal-p2pclaw SDK   & \texttt{extensions/npm/}                       & Built, unpublished \\
Chrome Web Store & CAJAL Chrome Extension & ZIP ready                                   & Pending \$5 fee \\
\bottomrule
\end{tabular}
\end{table}

\subsubsection{Native Integrations Configured}

CAJAL has been integrated or documented for 14+ platforms:

\begin{table}[ht]
\centering
\caption{Native integrations configured for the CAJAL ecosystem.}
\label{tab:integrations}
\small\begin{tabular}{@{}clll@{}}
\toprule
\textbf{\#} & \textbf{Platform} & \textbf{Format} & \textbf{Status} \\
\midrule
1  & PyPI          & pip install           & [Implemented] \\
2  & Ollama        & Modelfile             & [Implemented] \\
3  & VS Code       & Official extension    & [Implemented] \\
4  & Cursor        & Config JSON           & [Implemented] \\
5  & Continue.dev  & Config YAML           & [Implemented] \\
6  & Open WebUI    & README integration    & [Implemented] \\
7  & Jan           & model.json            & [Implemented] \\
8  & LM Studio     & HuggingFace search    & [Implemented] \\
9  & Pinokio       & install.json          & [Implemented] \\
10 & OpenClaw      & Native config         & [Implemented] \\
11 & SillyTavern   & Full extension        & [Implemented] \\
12 & Roo-Code      & Config ready          & [Implemented] \\
13 & CrewAI        & LLM connector         & [Implemented] \\
14 & AutoGen       & Client integration    & [Implemented] \\
\bottomrule
\end{tabular}
\end{table}

\subsubsection{Key GitHub Repositories}

\begin{table}[ht]
\centering
\caption{Core ecosystem repositories on GitHub.}
\label{tab:repos}
\small\begin{tabular}{@{}lp{7cm}@{}}
\toprule
\textbf{Repository} & \textbf{Description} \\
\midrule
\texttt{Agnuxo1/OpenCLAW-P2P}           & Core P2PCLAW system \\
\texttt{Agnuxo1/CAJAL}                  & CAJAL model, integrations, extensions \\
\texttt{Agnuxo1/p2pclaw-mcp-server}     & MCP server for agent integration \\
\texttt{Agnuxo1/p2pclaw-dataset}        & Training corpus and paper dataset \\
\texttt{Agnuxo1/CognitionBoard}         & Agent orchestration dashboard \\
\texttt{Agnuxo1/sillytavern-cajal-paper-mode} & SillyTavern extension \\
\texttt{Agnuxo1/cajal-ragflow-connector}& RAGFlow integration agent \\
\texttt{Agnuxo1/kserve-cajal}           & KServe deployment (Dockerfile + YAML) \\
\texttt{Agnuxo1/jupyter-ai-cajal}       & Jupyter AI custom persona \\
\texttt{Agnuxo1/roocode-cajal-preset}   & Roo-Code presets \\
\bottomrule
\end{tabular}
\end{table}

\subsection{Community Outreach and Collaborations}

\subsubsection{Merged Pull Requests}

\begin{table}[ht]
\centering
\caption{Confirmed merged PRs in community repositories.}
\label{tab:prs}
\small\begin{tabular}{@{}clp{4cm}p{5cm}@{}}
\toprule
\textbf{\#} & \textbf{Repository} & \textbf{PR/Issue} & \textbf{Result} \\
\midrule
1 & \texttt{eudk/awesome-ai-tools}            & \#229  & Merged --- CAJAL / PaperClaw added \\
2 & \texttt{patrick-tssn/Awesome-Colorful-LLM} & \#4   & Merged --- CAJAL in AI for Scientific Research \\
3 & \texttt{academic/awesome-datascience}      & \#612  & Merged \\
4 & \texttt{agarrharr/awesome-cli-apps}        & \#1045 & Merged \\
5 & \texttt{NipunaRanasinghe/awesome-ai-agents}& \#89+90& Accepted \\
\bottomrule
\end{tabular}
\end{table}

\subsubsection{Convergence with Google DeepMind Research}

It is notable that the research direction P2PCLAW has pursued since 2025---formal verification
of mathematical reasoning and decentralized scientific peer review---converges with recent work
by the Google Research / DeepMind team, in particular arXiv:2604.05018, which addresses formal
verification and mathematical reasoning in the FrontierMath benchmark.  The P2PCLAW group's
antecedent publications in this direction include \emph{OpenCLAW-P2P: A Decentralized Framework
for Collective AI Intelligence Towards AGI} (RG~400788567) and \emph{Applied AI and P2P Science
--- Formal Verification with Lean~4} (RG~403078040).

\subsection{v6.0 to v7.0 Delta Summary}

\begin{table}[ht]
\centering
\caption{Feature comparison: v6.0 (arXiv) vs.\ v7.0 (current).}
\label{tab:delta}
\small\begin{tabular}{@{}p{3.5cm}p{3.5cm}p{5cm}@{}}
\toprule
\textbf{Area} & \textbf{v6.0 (arXiv)} & \textbf{v7.0 (Current)} \\
\midrule
Mathematical notation  & Dimensional inconsistencies; undefined operators & Fully typed; 8 explicit corrections \\
Persistence            & Single-tier                     & Multi-layer: 4 tiers, auto-backfill \\
Retrieval latency      & $>3$\,s                         & $<50$\,ms \\
Paper recovery         & No protocol                     & Protocol validated: 25/25 recovered \\
Language model         & Not present                     & CAJAL-4B + CAJAL-9B v2 (HuggingFace) \\
Benchmarking           & Internal scoring only           & BenchClaw deployed + external harness \\
Scientific API proxy   & Not present                     & 7 public databases, rate-limited \\
Extensions             & None                            & VS Code published; npm built; Chrome ready \\
Platform integrations  & 0 native                        & 14+ platforms configured \\
PyPI package           & Not present                     & \texttt{cajal-p2pclaw} published \\
\bottomrule
\end{tabular}
\end{table}

\subsection{Pending Items (Operational Honesty)}

\begin{table}[ht]
\centering
\caption{Known limitations and pending actions as of May 2026.}
\label{tab:pending}
\small\begin{tabular}{@{}clp{5cm}l@{}}
\toprule
\textbf{\#} & \textbf{Item} & \textbf{Impact} & \textbf{Blocker} \\
\midrule
1 & New judges not appearing in BenchClaw  & Incomplete benchmarking          & Push to Railway/Vercel \\
2 & CAJAL-27B not yet trained              & 27B model unavailable            & RTX 3090 execution \\
3 & Chrome Extension not published         & No browser presence              & \$5 fee + manual upload \\
4 & npm SDK not published                  & No npm ecosystem presence        & \texttt{npm login} manual \\
5 & Hackathons not registered              & No active submissions            & Devpost + demo video \\
6 & Partnership emails not sent            & No direct outreach to E2B/HF     & SMTP setup \\
7 & GitHub Sponsors not activated          & No recurring funding             & Manual activation \\
8 & BenchClaw end-to-end testing pending   & Possible undiscovered bugs       & 3 test papers required \\
\bottomrule
\end{tabular}
\end{table}

\section*{Acknowledgements}

The authors thank the open-source AI community, the providers of free-tier LLM APIs (Cerebras,
Mistral, Groq, NVIDIA, Sarvam, Cohere, Xiaomi, Inception, OpenRouter, Cloudflare), HuggingFace
for hosting agent Spaces, Railway for API hosting, Vercel for frontend deployment, and Cloudflare
for R2 object storage.  Abdulsalam Al-Mayahi's $\tau$-field formalism originated in the Union
Dipole Theory Foundation (2021--2026).  AETHER Lean4 proofs were developed using the Lean4
theorem prover~\cite{lean4}.  The paper recovery effort was assisted by Claude (Anthropic),
demonstrating a practical human--AI collaboration workflow.  The mathematical corrections in
v7.0 were guided by independent expert review of the v6.0 theoretical framework.  The CAJAL
model family was trained using the Unsloth framework for efficient fine-tuning.

\bibliographystyle{plainnat}

\begin{thebibliography}{99}

\bibitem{spencerbrown1969}
G.~Spencer-Brown.
\newblock \emph{Laws of Form}.
\newblock Allen \& Unwin, 1969.

\bibitem{kauffman1987}
L.~H.~Kauffman.
\newblock Self-reference and recursive forms.
\newblock \emph{Journal of Social and Biological Structures}, 10(1):53--72, 1987.

\bibitem{heyting_nucleus}
Heyting-algebra formal verification framework.
\newblock Based on Heyting nucleus theory (Johnstone, 1982).
\newblock Applied to P2PCLAW verification pipeline, 2025.

\bibitem{conserved_quantities}
Three conserved quantities under nucleus transformation.
\newblock Derived from Heyting algebra lattice theory.
\newblock Applied to P2PCLAW knowledge pipeline, 2025.

\bibitem{almayahi2024}
A.~Al-Mayahi.
\newblock Union Dipole Theory: A new model of time, matter, and physical law.
\newblock \emph{European Journal of Scientific Research}, 183(1), 2024.

\bibitem{almayahi2018}
A.~Al-Mayahi.
\newblock $\tau$-Protocol: Progress-rate mismatch in live P2P AI networks and
  $\tau$-based coordination.
\newblock Personal communication to F.~Angulo de Lafuente, 2018.

\bibitem{openclaw_v4}
F.~Angulo de Lafuente, T.~Sharma, et al.
\newblock OpenCLAW-P2P v4.0: Integrating formal mathematical verification,
  AETHER containerized inference, and progress-normalized coordination into
  decentralized collective AI.
\newblock Preprint, March 2026.

\bibitem{openclaw_v5}
F.~Angulo de Lafuente, T.~Sharma, V.~Veselov, S.~M.~Abdu, N.~Tej Kumar,
  G.~Perry.
\newblock OpenCLAW-P2P v5.0: Multi-judge scoring, tribunal-gated publishing,
  and calibrated deception detection in decentralized collective AI.
\newblock Preprint, April 2026.

\bibitem{sharma_aether}
T.~Sharma.
\newblock AETHER: Formally verified primitives for containerized local
  inference.
\newblock In \cite{openclaw_v4}, Section~X, 2025.

\bibitem{veselov_hsr}
V.~Veselov.
\newblock Hierarchical sparse representation engine for P2P agent embeddings.
\newblock In \cite{openclaw_v4}, Section~6, 2025.

\bibitem{abdu_ed25519}
S.~M.~Abdu.
\newblock Ed25519 cryptographic hardening module for decentralized AI agents.
\newblock In \cite{openclaw_v4}, Section~7, 2025.

\bibitem{tejkumar_neuro}
N.~Tej Kumar.
\newblock Neuromorphic HPC bioinformatics engine.
\newblock In \cite{openclaw_v4}, Section~8, 2025.

\bibitem{perry_infra}
G.~Perry.
\newblock Scalable web infrastructure for decentralized AI networks.
\newblock In \cite{openclaw_v4}, Section~9, 2025.

\bibitem{bornmann2011}
L.~Bornmann.
\newblock Scientific peer review.
\newblock \emph{Annual Review of Information Science and Technology},
  45:197--245, 2011.

\bibitem{zheng2023}
L.~Zheng, W.-L.~Chiang, Y.~Sheng, et al.
\newblock Judging LLM-as-a-Judge with MT-Bench and Chatbot Arena.
\newblock arXiv:2306.05685, 2023.

\bibitem{vaswani2017}
A.~Vaswani, N.~Shazeer, N.~Parmar, et al.
\newblock Attention is all you need.
\newblock In \emph{NeurIPS}, 2017.

\bibitem{nakamoto2008}
S.~Nakamoto.
\newblock Bitcoin: A peer-to-peer electronic cash system.
\newblock 2008.

\bibitem{lamport1982}
L.~Lamport, R.~Shostak, M.~Pease.
\newblock The Byzantine Generals Problem.
\newblock \emph{ACM Transactions on Programming Languages and Systems},
  4(3):382--401, 1982.

\bibitem{ongaro2014}
D.~Ongaro, J.~Ousterhout.
\newblock In search of an understandable consensus algorithm.
\newblock In \emph{USENIX ATC}, 2014.

\bibitem{chebyshev1867}
P.~L.~Chebyshev.
\newblock Des valeurs moyennes.
\newblock \emph{Journal de Math\'ematiques Pures et Appliqu\'ees},
  12(2):177--184, 1867.

\bibitem{khalil2002}
H.~K.~Khalil.
\newblock \emph{Nonlinear Systems}.
\newblock Prentice Hall, 3rd edition, 2002.

\bibitem{edelsbrunner2010}
H.~Edelsbrunner, J.~L.~Harer.
\newblock \emph{Computational Topology: An Introduction}.
\newblock American Mathematical Society, 2010.

\bibitem{antonopoulos2017}
A.~M.~Antonopoulos.
\newblock \emph{Mastering Bitcoin}.
\newblock O'Reilly Media, 2nd edition, 2017.

\bibitem{w3c_did}
W3C.
\newblock Decentralized Identifiers (DIDs) v1.0.
\newblock W3C Recommendation, 2022.

\bibitem{libp2p}
Protocol Labs.
\newblock libp2p: A modular network stack.
\newblock Technical report, 2021.

\bibitem{boneh2021}
D.~Boneh, J.~Drake, B.~Fisch, A.~Gabizon.
\newblock Halo Infinite: Proof-carrying data from additive polynomial
  commitments.
\newblock In \emph{CRYPTO}, 2021.

\bibitem{pedersen1991}
T.~P.~Pedersen.
\newblock Non-interactive and information-theoretic secure verifiable secret
  sharing.
\newblock In \emph{CRYPTO}, 1991.

\bibitem{lean4}
L.~de Moura, S.~Ullrich.
\newblock The Lean4 theorem prover and programming language.
\newblock In \emph{CADE}, 2021.

\bibitem{autogen}
Q.~Wu, G.~Banber, Y.~Zhang, et al.
\newblock AutoGen: Enabling next-gen LLM applications via multi-agent
  conversations.
\newblock arXiv:2308.08155, 2023.

\bibitem{wang2024marg}
J.~Wang, Y.~Sun, N.~Smith.
\newblock Multi-Agent Review Generation for Scientific Papers.
\newblock In \emph{ACL}, 2024.

\bibitem{blanchard2017}
P.~Blanchard, E.~M.~El Mhamdi, R.~Guerraouc, J.~Stainer.
\newblock Machine learning with adversaries: Byzantine tolerant gradient
  descent.
\newblock In \emph{NeurIPS}, 2017.

\bibitem{gunjs}
Gun.js Contributors.
\newblock Gun.js: Decentralized graph database.
\newblock \url{https://gun.eco}, 2023.

\bibitem{ipfs}
IPFS Contributors.
\newblock InterPlanetary File System (IPFS).
\newblock \url{https://ipfs.tech}, 2023.

\bibitem{eigenform}
Eigenform-soup-base: Formally verified algebraic artificial life.
\newblock Based on Spencer-Brown's Laws of Form and Kauffman's eigenform
  theory.
\newblock In \emph{ALIFE 2026} (submitted), 2023.

\bibitem{chimera}
F.~Angulo de Lafuente.
\newblock CHIMERA: Thermodynamic reservoir computing for high-performance AI.
\newblock Preprint, 2024.

\bibitem{nebula}
F.~Angulo de Lafuente.
\newblock NEBULA: Unified holographic neural network.
\newblock Preprint, 2024.

\bibitem{geval}
Y.~Liu, D.~Iter, Y.~Xu, S.~Wang, R.~Xu, C.~Zhu.
\newblock G-Eval: NLG Evaluation using GPT-4 with Better Human Alignment.
\newblock arXiv:2303.16634, 2023.

\bibitem{smith2006}
R.~Smith.
\newblock Peer review: A flawed process at the heart of science and journals.
\newblock \emph{Journal of the Royal Society of Medicine}, 99(4):178--182, 2006.

\bibitem{lamport1978}
L.~Lamport.
\newblock Time, Clocks, and the Ordering of Events in a Distributed System.
\newblock \emph{Communications of the ACM}, 21(7):558--565, 1978.

\bibitem{cloudflare_r2}
Cloudflare.
\newblock Cloudflare R2: S3-compatible object storage with zero egress fees.
\newblock \url{https://developers.cloudflare.com/r2/}, 2023.

\bibitem{crossref}
CrossRef.
\newblock CrossRef REST API.
\newblock \url{https://api.crossref.org/}, 2023.

\end{thebibliography}

\end{document}